\documentclass{article}
\usepackage[utf8]{inputenc}

\usepackage[margin=1in]{geometry}

\usepackage{microtype}
\usepackage{graphicx}
\usepackage{subfigure}
\usepackage{booktabs} %

\usepackage{amsmath}
\RequirePackage{mathtools}
\RequirePackage{amsfonts}
\RequirePackage{dsfont}
\RequirePackage{delimset}

\newcommand{\prn}[2][*]{\brk#1{#2}}
\newcommand{\brc}[2][*]{\brk[c]#1{#2}}
\newcommand{\underbracetxt}[2]{\underset{#2}{\underbrace{#1}}}

\newcommand{\ceil}[2][*]{\delim\lceil\rceil#1{#2}}
\newcommand{\floor}[2][*]{\delim\lfloor\rfloor#1{#2}}
\newcommand{\indEvent}[2][*]{\mathds{1}{\brc[#1]{#2}}}

\newcommand{\tr}[2][*]{\mathrm{Tr}\prn[*]{#2}}

\newcommand{\EE}[2][]{\mathbb{E}_{#1}{#2}}
\newcommand{\EEBrk}[2][*]{\mathbb{E}\delim{[}{]}#1{#2}}
\newcommand{\RR}[1][]{\mathds{R}^{#1}}

\newcommand{\PP}[2][*]{\mathbb{P}\prn[#1]{#2}}

\newcommand{\gaussDist}[2]{\mathcal{N}\prn[0]{#1, #2}}

\DeclareMathOperator*{\argmin}{arg\,min}

\DeclarePairedDelimiterX\setDef[1]\lbrace\rbrace{#1}
\newcommand{\seqDef}[3]{\brc{#1}_{#2}^{#3}}

\usepackage{amsthm}
\usepackage{amssymb}
\usepackage{thmtools}

\declaretheoremstyle[
	    spaceabove=\topsep, 
	    spacebelow=\topsep, 
	    bodyfont=\normalfont\itshape,
    ]{theorem}

\declaretheorem[style=theorem,name=Theorem]{theorem}

\declaretheoremstyle[
	    spaceabove=\topsep, 
	    spacebelow=\topsep, 
	    bodyfont=\normalfont,
    ]{definition}

\declaretheoremstyle[
        spaceabove=\topsep, 
        spacebelow=\topsep, 
        bodyfont=\normalfont,
        notefont=\normalfont\bfseries,
        notebraces={}{},
        qed=$\blacksquare$, 
    ]{proofstyle}
\declaretheorem[style=proofstyle,numbered=no,name=Proof]{proof}

\declaretheorem[style=theorem,sibling=theorem,name=Lemma]{lemma}

\declaretheorem[style=theorem,numbered=no,name=Theorem]{theorem*}
\declaretheorem[style=theorem,numbered=no,name=Lemma]{lemma*}
\declaretheorem[style=theorem,numbered=no,name=Corollary]{corollary*}
\declaretheorem[style=theorem,numbered=no,name=Proposition]{proposition*}
\declaretheorem[style=theorem,numbered=no,name=Claim]{claim*}
\declaretheorem[style=theorem,numbered=no,name=Fact]{fact*}
\declaretheorem[style=theorem,numbered=no,name=Observation]{observation*}
\declaretheorem[style=theorem,numbered=no,name=Conjecture]{conjecture*}

\declaretheorem[style=definition,sibling=theorem,name=Definition]{definition}

\declaretheorem[style=definition,numbered=no,name=Definition]{definition*}
\declaretheorem[style=definition,numbered=no,name=Remark]{remark*}
\declaretheorem[style=definition,numbered=no,name=Example]{example*}
\declaretheorem[style=definition,numbered=no,name=Question]{question*}
\declaretheorem[style=definition,numbered=no,name=Assumption]{assumption*}

\usepackage[round]{natbib}
\usepackage[colorlinks=true,allcolors=blue,hyperfootnotes=false]{hyperref}
\usepackage[capitalise]{cleveref}

\usepackage{bm}
\usepackage{enumitem}
\usepackage{algorithm}
\usepackage[noend]{algpseudocode}

\title{Logarithmic Regret for Learning \\ Linear Quadratic Regulators Efficiently}

\author{%
Asaf Cassel%
\thanks{School of Computer Science, Tel Aviv University; \texttt{acassel@mail.tau.ac.il}.}
\and
Alon Cohen%
\thanks{Google Research, Tel Aviv; \texttt{aloncohen@google.com}.}
\and
Tomer Koren%
\thanks{School of Computer Science, Tel Aviv University; \texttt{tkoren@tauex.tau.ac.il}.}
}

\newcommand{\poly}[1]{\mathrm{poly}\prn{#1}}

\newcommand{\costMatLower}{\alpha_0}
\newcommand{\costMatUpper}{\alpha_1}
\newcommand{\noiseStd}{\sigma}
\newcommand{\systemBound}{\vartheta}
\newcommand{\optCostBound}{\nu}
\newcommand{\startCostBound}{\nu_0}

\newcommand{\KstarLowerBound}{{\mu}_{\star}}
\newcommand{\KstarLowerBoundGuess}{{\mu}_{0}}
\newcommand{\AlgKstarLowerBound}[1][i]{\mu_{#1}}

\newcommand{\Astar}{A_{\star}}
\newcommand{\Bstar}{B_{\star}}
\newcommand{\Kstar}{K_{\star}}

\newcommand{\Jstar}{J_{\star}}

\newcommand{\Jof}[1]{J\prn{#1}}
\newcommand{\regret}[1][T]{R_{#1}}

\newcommand{\goodEventA}[1][]{\mathcal{E}_{A_{#1}}}
\newcommand{\goodEventB}[1][]{\mathcal{E}_{B_{#1}}}

\newcommand{\olsEventB}{\goodEventB[ols]}
\newcommand{\xNoiseExploreEventB}{\goodEventB[x]}
\newcommand{\xNoiseBoundEventB}{\goodEventB[w]}
\newcommand{\uNoiseExploreEventB}{\goodEventB[u]}
\newcommand{\uNoiseBoundEventB}{\goodEventB[\eta]}

\newcommand{\olsEventA}{\goodEventA[ols]}
\newcommand{\xNoiseExploreEventA}{\goodEventA[x]}
\newcommand{\xNoiseBoundEventA}{\goodEventA[w]}

\newcommand{\rechtEps}{\varepsilon_0}
\newcommand{\rechtConst}{C_0}

\newcommand{\tAbort}{\tau_{\text{abort}}}
\newcommand{\tWarmup}{\ti[\nStart]}
\newcommand{\xBreak}{x_b}
\newcommand{\tBase}{\tau_0}
\newcommand{\ti}[1][i]{\tau_{#1}}

\newcommand{\Vu}[1]{V^{u}_{#1}}
\newcommand{\Vx}[1]{V^{x}_{#1}}

\newcommand{\Kopt}{\mathcal{K}}
\newcommand{\KoptOf}[3][Q, R]{\Kopt\prn{#2, #3, #1}}

\newcommand{\uNoise}[1][t]{\eta_{#1}}
\newcommand{\uVirt}[1][t]{\tilde{u}_{#1}}

\newcommand{\dEstB}[1][t]{\Delta_{B_{#1}}}
\newcommand{\dEstA}[1][t]{\Delta_{A_{#1}}}

\newcommand{\nT}{n_T}
\newcommand{\nStart}{n_s}

\renewcommand{\Pr}{\mathbb{P}}
\newcommand{\TV}[2]{\text{TV}\delim(){#1, #2}}
\newcommand{\KL}[2]{\text{KL}\delim(){#1 \;\|\; #2}}

\begin{document}
\maketitle

\begin{abstract}
We consider the problem of learning in Linear Quadratic Control systems whose transition parameters are initially unknown.
Recent results in this setting have demonstrated efficient learning algorithms with regret growing with the square root of the number of decision steps. 
We present new efficient algorithms that achieve, perhaps surprisingly, regret that scales only (poly)logarithmically with the number of steps in two scenarios: when only the state transition matrix $A$ is unknown, and when only the state-action transition matrix $B$ is unknown and the optimal policy satisfies a certain non-degeneracy condition.
On the other hand, we give a lower bound that shows that when the latter condition is violated, square root regret is unavoidable.
\end{abstract}

\section{Introduction}

The linear-quadratic regulator model (LQR) is a classic model in optimal control theory.
In this model, the dynamics of the environment are given as 
\[
    x_{t+1} = \Astar x_t + \Bstar u_t + w_t,
\]
where $x_t$ and $u_t$ are the state and the action vectors at time $t$, $\Astar$ and $\Bstar$ are transition matrices, and $w_t$ is a zero-mean i.i.d.\ Gaussian noise.
The cost function is quadratic in both the state and the action.
An interesting property of LQR systems is that a linear control policy minimizes the cost while keeping the system at a steady-state (stable) position.

In this work, we study the problem of designing an adaptive controller that regulates the system while learning its parameters. 
This problem has recently been approached through the lens of regret minimization, beginning in the work of \citet{abbasi2011regret} that established an $O(\sqrt{T}$) regret bound for this setting albeit with a computationally inefficient algorithm. The problem of designing an efficient algorithm that enjoys $O(\sqrt{T})$ was later resolved by \citet{cohen2019learning} and \citet{mania2019certainty}.
The former work relied on the ``optimism in the face of uncertainty'' principle and a reduction to an online semi-definite problem, and the latter work used a simpler greedy strategy. 

Following this line of work, it has been believed that an $O(\sqrt{T})$ regret is tight for the problem. This appears natural as it is the typical rate for many imperfect information (bandit) optimization~problems (e.g., \citealp{shamir2013complexity}).%
\footnote{More precisely, this is very often the regret rate in bandit problems with no ``gap'' assumptions regarding the difference between the best and second-best actions/policies.} 
On the other hand, one could suspect that better, polylogarithmic regret bounds, are possible in the LQR setting thanks to the strongly convex structure of the cost functions. 
Often in optimization, this structure gives rise to faster convergence/regret rates, and indeed, in a recent work, \citet{agarwal2019logarithmic} have demonstrated that such fast rates are attainable in the related, yet full-information online LQR problem endowed with any strongly convex loss functions.

In this paper, we show two interesting scenarios of learning unknown LQR systems in which an expected regret of~$O(\log^2 T)$ is, in fact, achievable.
In the first, we assume that the matrix $\Bstar$ is known and show that polylogarithmic regret can be attained by harnessing the intrinsic noise in the system dynamics for exploration.
In the second, we assume that $\Astar$ is known and that the optimal control policy $\Kstar$ is given by a full-rank matrix. 
Both results are attained using simple and efficient algorithms whose runtime per time step is polynomial in the natural parameters of the problem.

We complement our results with a lower bound showing that our assumptions are indeed necessary for obtaining improved regret guarantees. 
Specifically, we show that when $\Bstar$ is unknown and the optimal policy $\Kstar$ is near-degenerate (i.e., with very small singular values), any online algorithm, whether efficient or not, must suffer at least $\Omega(\sqrt{T})$ regret.
To the best of our knowledge, this is the first $\Omega(\sqrt{T})$ lower bound for learning linear quadratic regulators (that particularly holds even when the learner knows the entire set of system parameters but the matrix~$\Bstar$, and even in a single-input single-output scenario).
Concurrent to this work, \citet{simchowitz2020naive} suggest a different lower bound construction that relies on uncertainty in both $\Astar$ and $\Bstar$, and thus does not contradict our positive findings.

\subsection{Setup: Learning in LQR}
We consider the problem of regret minimization in the LQR model.
At each time step $t$, a state $x_t \in \RR[d]$ is observed and action $u_t \in \RR[k]$ is chosen.
The system evolves according to
\begin{equation*}
	x_{t+1} = \Astar x_t + \Bstar u_t + w_t,
\end{equation*}
where the state-state $\Astar \in \RR[d \times d]$ and state-action $\Bstar \in \RR[d \times k]$ matrices form the transition model and the $w_t$ are i.i.d.~noise terms, each is a zero mean Gaussian with covariance matrix $\noiseStd^2 I$. 
At time $t$, the instantaneous cost is 
$$
    c_t = x_t^T Q x_t + u_t^T R u_t,
$$ 
where $Q,R \succ 0$ are positive definite. 

A policy of the learner is a mapping from a state $x \in \RR[d]$ to an action $u \in \RR[k]$ to be taken at that state.
Classic results in linear control establish that, given the system parameters $\Astar,\Bstar,Q$ and $R$, the optimal policy is a linear mapping from the state space $\RR[d]$ to the action space $\RR[k]$ in an infinite-horizon setup.
We thus consider policies of the form $u_t = K x_t$ and define the infinite horizon expected cost,
\begin{equation*}
	\Jof{K} = \lim_{T \to \infty} \frac{1}{T} \EEBrk{\sum_{t=1}^{T} x_t^T \prn{Q + K^T R K} x_t},
\end{equation*}
where the expectation is taken with respect to the random noise variables $w_t$. Let $\Kstar = \argmin_K\Jof{K}$ be an (unique) optimal policy and $\Jstar = \Jof{\Kstar}$ denote the optimal infinite horizon expected cost, which are both well defined under mild assumptions.%
\footnote{These hold under standard, very mild controllablity assumptions (see \citealp{bertsekas1995dynamic}) that we implicitly assume throughout.}
We are interested in minimizing the \emph{regret} over~$T$ decision rounds, defined as
\begin{equation*}
    \regret
    =
    \sum_{t=1}^{T} \prn[!]{x_t^T Q x_t + u_t^T R u_t - \Jstar}
    .
\end{equation*}
We focus on the setting where the learner does not have a full a-priori description of the transition parameters $\Astar$ and $\Bstar$, and has to learn them while controlling the system and minimizing the regret.

Throughout, we assume that the learner has knowledge of the cost matrices $Q$ and $R$, and that there are constants $\costMatLower,\costMatUpper>0$ such that 
$$
    \norm{Q}, \norm{R} \leq \costMatUpper,
    \text{ and }
    \norm{Q^{-1}}, \norm{R^{-1}} \leq \costMatLower^{-1}
    .
$$
We further assume that the learner has bounds on the transition matrices, as well as on the optimal cost; that is, there are known constants $\systemBound, \optCostBound>0$ such that
$$
    \norm{\Astar},\norm{\Bstar} \le \systemBound,
    \text{ and }
    \Jstar \le \optCostBound
    .
$$
Finally, we assume that there is a known stable (not necessarily optimal) policy $K_0$ and $\startCostBound>0$ such that $\Jof{K_0} \le \startCostBound$.%
\footnote{Regarding the necessity of this assumption, see the discussion in \citet{mania2019certainty,cohen2019learning}.}

\subsection{Main results}

Our first result focuses on the case where the state-action transition matrix $\Bstar$ is known (but the matrix $\Astar$ is unknown).

\begin{theorem} \label{thm:unknownAregret}
There exists an efficient online algorithm (see~\cref{alg:A} in~\cref{sec:algA}) that, given the matrix $\Bstar$ as input, has expected regret
\begin{equation*}
    \EEBrk{\regret}
    =
    \poly{\costMatLower^{-1}, \costMatUpper, \systemBound, \optCostBound, \startCostBound, d, k} \, O(\log^2 T)
    .
\end{equation*}
\end{theorem}

Next, we consider the dual setup in which only the state-state matrix $\Astar$ is known. Here we require an additional non-degeneracy assumption for obtaining polylogarithmic regret.

\begin{theorem} \label{thm:unknownBregret}
Suppose that the optimal policy of the system satisfies $\Kstar \Kstar^T \succeq \KstarLowerBound I$ for some constant $\KstarLowerBound>0$ that is unknown to the learner.
Then there exists an efficient online algorithm (see~\cref{alg:B} in~\cref{sec:algB}) that, given the matrix $\Astar$ as input, has expected regret
\begin{equation*}
    \EEBrk{\regret} 
    =
    \poly{\KstarLowerBound^{-1}, \costMatLower^{-1}, \costMatUpper, \systemBound, \optCostBound, \startCostBound, d, k} \, O(\log^2 T)
    .
\end{equation*}
\end{theorem}

Finally, we show that our assumption regarding the non-degeneracy of the optimal policy $\Kstar$ is necessary.
Our next result shows that without it, the expected regret of any algorithm is unavoidably at least $\Omega(\sqrt{T})$, even in simple one-dimensional (single input, single output) systems.

\begin{theorem} \label{thm:lb-main}
For any learning algorithm and any $\sigma>0$, there exists an LQR system (in dimensions $d=k=1$) which is stabilized by the policy $K_0=0$ and for which $\costMatUpper=\costMatLower=1$, $\systemBound=1$ and $\optCostBound=2\sigma^2$, such that the expected regret of the algorithm is at least
$
    \Omega(\sigma^2 \sqrt{T})
    .
$
This is true even if the algorithm receives the matrix $\Astar$ as input.
\end{theorem}

\subsection{Discussion}

Our results could be interpreted as a proof-of-concept that faster, polylogarithmic rates for learning in LQRs are possible under more limited uncertainty assumptions.
This is perhaps surprising in light of the aforementioned work of \citet{shamir2013complexity}, that established $\Omega(\sqrt{T})$ regret lower bounds for online (bandit) optimization, even with quadratic and strongly convex objectives (as is the case in our LQR setup).
The questions of whether polylogarithmic regret guarantees
are possible under more general
uncertainty of the system parameters,
as well as whether the squared dependence on $\log T$ is indeed tight, remain open. 
Our lower bound, however, shows that more assumptions are required for obtaining stronger positive results.

Our results focused on the \emph{expected} regret compared to the \emph{infinite-horizon} performance of the optimal policy $\Kstar$.
As far as we know, this is the first analysis that bounds the regret in expectation rather than in high-probability (and without additional assumptions, e.g., as in \citealp{ouyang2017control}).
Indeed, in previous analyses we are aware of, there was always a small probability where the algorithm fails and incurs very large (possibly exponentially large) regret.
Here, we address this low-probability event by employing a novel ``abort procedure'' when our algorithms suspect the system has been destabilized; this ensures that the expected regret remains controlled.
The question of whether our regret bounds hold with high probability remains for future investigation. 
We remark that in the analogous multi-armed bandit setting, it is well-known that the logarithmic expected regret bounds of UCB-type algorithms can be converted into high probability ones, and so it is a natural question whether the same holds for LQRs.

We also remark that the infinite-horizon cost of the optimal policy can be easily replaced in the definition of the regret with the finite-time cost of $\Kstar$ (up to additional additive low order terms). This is since the expected costs of any (strongly) stable policy converge exponentially fast to its expected steady-state cost.
One could also consider a different definition of the regret, akin to that of ``pseudo-regret'' in multi-armed bandits, where the learner has to commit at each time step to a linear policy $K_t$ and incurs its mean infinite-horizon cost, $\Jof{K_t}$. (This is the type of notion considered in several recent papers, e.g., \citealp{fazel2018global,malik2019derivative}.)
We note, however, that in the unbounded LQR setting there are subtleties that make this definition potentially weaker than the actual expected regret that we focus on; for example, the learner could choose $K_t$ so as to deliberately blow up the magnitude of the states and thereby boost the estimation rates of the unknown system parameters, but at the same time, $\Jof{K_t}$ would remain controlled and no significant penalty in the regret will be incurred.

\subsection{Related work}

The subject of linear quadratic optimal control has been studied extensively in control theory. Importantly, \citet{lai1982least} establish asymptotic convergence rates of system identification, while \citet{polderman1986necessity} show the necessity of said identification for optimal control. More generally, it is known that greedy control strategies only converge to a potentially large subset of the parameter space (see e.g., \citealp{kumar1983optimal,kumar1985survey,polderman1986note}), and that in the context of our assumptions this subset is a singleton containing only the true system parameters. While this may allude to our positive findings, the asymptotic nature of existing results makes them insufficient for establishing finite-time (polylogarithmic) regret guarantees, which are the focus of this work.

The topic of learning in linear control has been attracting considerable attention in recent years.
Since the early work of \citet{abbasi2011regret}, a long line of research has focused on obtaining improved regret bounds for learning in LQRs with a variety of algorithms~\citep{ibrahimi2012efficient,faradonbeh2017optimism,abeille2018improved,dean2018regret,faradonbeh2018input,cohen2019learning,abbasi2019politex,abbasi2019model}.
To the best of our knowledge, our results are the first to exhibit logarithmic regret rates for LQRs albeit in a more restrictive setting.

A closely related line of work considered a non-stochastic variant of online control in which the cost functions can change arbitrarily from round to round~\citep{cohen2018online,agarwal2019online,agarwal2019logarithmic}.
Other notable works have studied the sample complexity of estimating the unknown parameters of linear dynamical systems~\citep{dean2017sample,simchowitz2018learning,sarkar2019near}, improper prediction of linear systems~\citep{hazan2017learning,hazan2018spectral}, as well as model-free learning of LQRs via policy gradient methods~\cite{fazel2018global,malik2019derivative}.

\section{Preliminaries}

\subsection{Linear Quadratic Control}

We give a brief background on several basic properties and results in linear quadratic control that we require in the paper.
For a given LQR system $(A, B)$ with cost matrices $Q,R \succ 0$, the optimal (infinite horizon) feedback controller is given by

\vspace{-1em}
\begin{equation} \label{eq:Kopt}
	\KoptOf{A}{B}
	=
	- \prn{R + B^T P B}^{-1} B^T P A,
\end{equation}

where $P$ is the positive definite solution to the discrete Riccati equation
\begin{equation} \label{eq:riccati}
	P 
	= 
	Q + A^T P A - A^T P B \prn{R + B^T P B}^{-1} B^T P A
	.
\end{equation}
In particular, for the system $(\Astar,\Bstar)$ we have $\Kstar = \KoptOf{\Astar}{\Bstar}$.
For more background on linear control and derivation of the relations above, see \citet{bertsekas1995dynamic}.

The following lemma, proved in \citet{mania2019certainty}, relates the error in estimating a system's parameters to the deviation of the corresponding estimated controller from the optimal one. This relation is given in terms of cost as well as in terms of distance in operator norm.

\begin{lemma} \label{lemma:recht}
    There are explicit constants 
    $
    \rechtConst,\rechtEps = \mathrm{poly}(\costMatLower^{-1},\allowbreak \costMatUpper, \systemBound, \optCostBound, \startCostBound, d, k)
    $
    such that, for any $0 \leq \varepsilon \leq \rechtEps$ and matrices $A,B$ such that
    $\norm{A - \Astar} \leq \varepsilon$ and $\norm{B - \Bstar} \leq \varepsilon$,
    the policy $K = \KoptOf{A}{B}$ satisfies
    $$
    \Jof{K} - \Jstar \leq \rechtConst \varepsilon^2 ,
    \;\;\;\text{and}\;\;\;
    \norm{K - \Kstar} \leq \rechtConst \varepsilon.
    $$
\end{lemma}

Importantly, the lemma shows that the performance scales \emph{quadratically} in the estimation error. 
This served \citet{mania2019certainty} as a key feature in showing that an $\varepsilon$-greedy algorithm obtains $O(\sqrt{T})$ regret. 
Here, we use this lemma to show that considerably improved regret bounds are  achievable in certain scenarios.

Next, we recall the notion of strong stability \citep{cohen2018online}.
This is essentially a quantitative version of classic stability notions in linear control.
\begin{definition}[strong stability]
		A matrix $M$ is $(\kappa, \gamma)-$strongly stable (for $\kappa \ge 1$ and $0 < \gamma \le 1$) if there exists matrices $H \succ 0$ and $L$ such that $M = H L H^{-1}$ with $\norm{L} \le 1 - \gamma$ and $\norm{H}\norm{H^{-1}} \le \kappa$. A controller $K$ for the system $(A, B)$ is $(\kappa, \gamma)-$strongly stable if $\norm{K} \le \kappa$ and the matrix $A +BK$ is $(\kappa, \gamma)-$strongly stable.
\end{definition}

We remark that \citet{cohen2018online} also introduced the notion of sequential strong stability that is an analogous definition for an adaptive strategy that changes its linear policy over time. 
Here, we avoid this notion by ensuring that each linear policy is played in our algorithms for a sufficiently long duration. 

\subsection{Confidence bounds for least-squares estimation}

Our algorithms use regularized least squares methods in order to estimate the system parameters.
An analysis of this method for a general, possibly-correlated sample, was introduced in the context of linear bandit optimization~\citep{abbasi2011improved}, and was first used in the context of LQRs by~\citet{abbasi2011regret}.
We state the results in terms of a general sequence, since the estimation procedures differ between our two algorithms.

Let $\Theta_\star \in \RR[d \times m]$, $\seqDef{y_{t+1}}{t=1}{\infty} \in \RR[d]$, $\seqDef{z_t}{t=1}{\infty} \in \RR[m]$, $\seqDef{w_t}{t=1}{\infty} \in \RR[d]$ such that
$y_{t+1} = \Theta_\star z_t + w_t$, and $\seqDef{w_t}{t=1}{\infty}$ are i.i.d.~with distribution $\mathcal{N}(0, \noiseStd^2 I)$. 
Denote by
\begin{equation}
\label{eq:olsDef}
    \hat{\Theta}_t
    \in
    \argmin_{\Theta \in \RR[d \times m]} \left\{ \sum_{s=1}^{t-1} \norm{y_{t+1} - \Theta z_t}^2 + \lambda \norm{\Theta}_F^2 \right\}
    ,
\end{equation}
the regularized least squares estimate of $\Theta_\star$ with regularization parameter $\lambda$.
\begin{lemma}[\citealp{abbasi2011regret}] \label{lemma:parameterEst}
Let $V_t = \lambda I + \sum_{s=1}^{t-1}z_t z_t^T$ and $\Delta_t = \Theta_\star - \hat{\Theta}_t$. 
With probability at least $1 - \delta$, we have for all $t \ge 1$
\begin{equation*}
\tr{\Delta_t^T V_t \Delta_t}
\le
4 \noiseStd^2 d \log \prn{\frac{d}{\delta}\frac{\det \prn{V_t}}{\det \prn{V_1}}} + 2\lambda \norm{\Theta_\star}_F^2.
\end{equation*}
\end{lemma}

\section{Proofs and Algorithms}

In this section we present our algorithms and illustrate the main ideas of our upper and lower bounds.
The complete versions of the proofs are deferred to \cref{sec:proofsA,sec:proofsB,sec:lb-proofs}.

\subsection{Upper Bound for Unknown $\bm\Astar$}
\label{sec:algA}

We start with the setting where $\Astar$ is unknown, and show an efficient algorithm that achieves regret at most $O\prn{\log^2 T}$. 
To that end, we propose \cref{alg:A}. The algorithm begins by playing the stable controller $K_0$ for a $\tBase$-long warm-up period. It subsequently operates in phases whose length grows exponentially (quadrupling). Each phase begins by estimating the system parameters using \cref{eq:olsDef} and computing the greedy controller with respect to said parameters using \cref{eq:Kopt}. It then proceeds to play greedily as long as a fail condition is not reached.

\begin{algorithm}
	\caption{} \label{alg:A}
	\begin{algorithmic}[1]
		\State \textbf{input:} parameters $\tBase, \xBreak, \kappa, \lambda$, a strongly stable controller $K_0$, and the action-state transition matrix $\Bstar$.
		\State \textbf{initialize:} $\nT = \floor{\log_4(T/\tBase)}$, $\ti[\nT+1] = T+1$
		\State \textbf{set:} $\ti \gets \tBase 4^i$ for all $0 \le i \le \nT$.
		
        \For{$t = 1, \ldots, \tBase - 1$}
            \Comment{warm-up}
            \State \textbf{play} $u_t = K_0 x_t$.
        \EndFor
        
        \For{\textbf{phase} $i = 0, \ldots, n_T$}
            \Comment{main loop}
        
            \State $A_{\tau_i} = \argmin_A \sum_{s=1}^{\tau_i - 1} \norm{\prn{x_{s+1} - \Bstar u_s} - A x_s}^2 \!+\! \lambda \norm{A}_F^2$
			\State $K_{\tau_i} = \KoptOf{A_{\tau_i}}{\Bstar}$.
			
		    \For{$t = \tau_i, \ldots, \tau_{i+1} - 1$}
			    \If{$\norm{x_t}^2 > \xBreak$ \textbf{or} $\norm{K_{\tau_i}} > \kappa$}
				    \Comment{fail, abort}
				    \State \textbf{abort} and play $K_0$ forever.
			    \EndIf
			    
			    \State \textbf{play} $u_t = K_{\tau_i} x_t$.
			\EndFor
		\EndFor
	\end{algorithmic}
\end{algorithm}

We now give a quantified restatement of  \cref{thm:unknownAregret}.

\begin{theorem*}[\cref{thm:unknownAregret} restated] %
Suppose \cref{alg:A} is run with parameters
\begin{gather*}
    \kappa_0
    =
    \sqrt{\frac{\startCostBound}{\costMatLower \noiseStd^2}},
    \;\;
    \kappa
    =
    \sqrt{\frac{\optCostBound + \rechtEps^2 \rechtConst}{\costMatLower \noiseStd^2}},
    \;\;
    \tBase
    =
    \ceil{
    	\frac{80 d \lambda \prn{1 + \systemBound^2}}{\noiseStd^2 \rechtEps^2}
    	}
    ,
    \\
    \lambda
    =
    \xBreak
    =
    135 d \kappa^2 \noiseStd^2 \max\brc{\kappa_0^6, 4 \kappa^6} \log\prn{3T}
    .
\end{gather*}
Then for
$
    T
    \ge
    \poly{\costMatLower^{-1}, \costMatUpper, \systemBound, \optCostBound, \startCostBound, d, k}
$
we have
$
    \EEBrk{\regret}
    \le
    \poly{\costMatLower^{-1}, \costMatUpper, \systemBound, \optCostBound, \startCostBound, d, k}
    \log^2 T
    .
$
\end{theorem*}

    We start by quantifying a high probability event on which the regret of the algorithm is small. The event holds when the error of the algorithm's estimate of $\Astar$ scales as $t^{-1/2}$, the states are bounded, and all control policies generated by the algorithm are strongly-stable. 
    This is formally given by the following lemma.
    
    \begin{lemma}
    	\label{lemma:goodOperationA}
    	Let $\gamma = 1 \big/ 2\kappa^2$. With probability at least $1 - T^{-2}$,
    	\begin{enumerate}[nosep,label=(\roman*)]
    		\item $K_{\ti}$ is $(\kappa,\gamma)-$strongly stable, for all $0 \le i \le \nT$;
    		\item $\norm{x_t}^2 \le \xBreak$, for all $1 \le t \le T$;
    		\item $\norm{\dEstA[\ti]} \leq \rechtEps 2^{-i}, \text{ for all } 0 \le i \le \nT$.
    	\end{enumerate}
    \end{lemma}
    
    Here we give a sketch of the proof of \cref{lemma:goodOperationA}, deferring technical details to \cref{sec:proofsA}.
    
    \begin{proof}[(sketch)]
	    Consider \cref{lemma:parameterEst} with $z_t = x_t, y_{t+1} = x_{t+1} - \Bstar u_t$, $V_t = \lambda I + \sum_{s=1}^{t-1} x_s x_s^T$ and $\dEstA  = A_t - \Astar$, then we have with probability at least $1 - \tfrac13 T^{-2}$
	    \begin{equation} \label{eq:paramEstAV}
	        \tr{\dEstA^T V_t \dEstA}
	        \le
	        4 \noiseStd d \log \prn{3 d T^2 \frac{\det \prn{V_t}}{\det \prn{V_1}}} + 2\lambda d \systemBound^2,
	    \end{equation}
	    for all $t \ge 1$. Transforming \cref{eq:paramEstAV} into the desired bound requires that we bound $V_t$ from above and below.
	    In what follows we show $\norm{V_t} \le \lambda t$ on one hand, and $V_t \succeq \frac{\noiseStd^2 t}{40} I $ on the other hand.
	    Using the upper bound and choice of parameters, one can show that simplifying the right hand side of \cref{eq:paramEstAV} yields $\tr{\dEstA^T V_t \dEstA} \le \noiseStd^2 \rechtEps^2 \tBase / 40$. Complementing this with the lower bound gets us
	    \begin{align*}
		    \norm{\dEstA}^2
		    \le
		    \tr{\dEstA^T \dEstA}
		    \le
		    \frac{40}{\noiseStd^2 t}\tr{\dEstA^T V_t \dEstA}
		    \le
		    \frac{\rechtEps^2 \tBase}{t}
		    ,
	    \end{align*}
	    and taking the square root, we obtain the desired estimation error bound that indeed scales as $t^{-1/2}$ (up to logarithmic factors).
	    
	    For a lower bound on $V_t$, notice that the system noise $w_t$ ensures that we have a sufficient exploration of the state space. Formally, we have
	    $$
	    \EEBrk{V_t}
	    \succeq
	    \lambda I
	    +
	    \sum_{s=1}^{t-1} \EEBrk{x_s x_s^T}
	    \succeq
	    t \noiseStd^2 I
	    ,
	    $$
	    where we used $\EEBrk{x_s x_s^T} \succeq \EEBrk{w_s w_s^T} \succeq \sigma^2 I$ and $\lambda \ge \sigma^2$.
	    Applying a measure concentration argument yields the sought-after high-probability lower bound on $V_t$.
	    
	    Now, for an upper bound on $V_t$, notice that 
	    $$
	    \norm{V_t}
	    \le
	    \lambda
	    +
	    \sum_{s=1}^{t-1} \norm{x_s}^2
	    $$
	    thus it suffices to show that $\norm{x_t}^2 \le \xBreak = \lambda$. The proof of the lemma now follows inductively by the following argument. If the parameter estimation at time $\ti$ holds then $K_{\ti}$ is strongly-stable. This implies that the states throughout phase $i$ satisfy $\norm{x_t}^2 \le \xBreak$ which in turn implies the upper bound on $V_{\ti[i+1]}$. Thus we can bound the parameter estimation error at time $\ti[i+1]$. We note that the initial parameter estimation, i.e., at time $\tBase$, follows from the strong-stability of $K_0$ and by taking the warm-up duration $\tBase$ to be sufficiently long. 
	\end{proof}
    
	\begin{proof}[of \cref{thm:unknownAregret}]
	    Let $\goodEventA$ be the event where \cref{lemma:goodOperationA} hold, and notice that the algorithm does not abort on $\goodEventA$. Defining $J_i = \sum_{t = \ti}^{\ti[i+1] - 1} x_t^T \prn{Q + K_{\ti}^T R K_{\ti}} x_t$,
		we have the following decomposition of the regret:
		\begin{align*}
		    \EEBrk{\regret}
			=
			R_1 + R_2 + R_3 - T \cdot \Jstar,
		\end{align*}
		where 
		\begin{align*}
		    R_1 = \EEBrk{\indEvent{\goodEventA} \sum_{i=0}^{\nT} J_i};
		    \qquad
		    R_2 = \EEBrk{\indEvent{\goodEventA^c}\sum_{t=\tBase}^{T} c_t};
		    \qquad 
		    R_3 = \EEBrk{\sum_{t=1}^{\tBase - 1} c_t},
		\end{align*}
		are the costs due to success, failure, and warm-up respectively. 
		We now bound each of $R_1,R_2,R_3$ to conclude the proof.
	
		Starting with $R_1$, the following lemma uses the strong-stability of $K_{\ti}$ (whenever $\goodEventA$ holds) to show that $J_i$ is closely related to the steady-state cost of $K_{\tau_i}$. 
		
		\begin{lemma} \label{lemma:JiBoundA}
		    Fix some $i$ such that $0 \le i \le \nT$, and define the event
		    $
		    	E_i
		    	=
		    	\brc{
		    	    \norm{\dEstA[\ti]}
		    	    \le
		    	    \rechtEps 2^{-i}
		    	}
		    	.
			$
			We have 
			\[
			    \EEBrk{\indEvent{\goodEventA} J_i}
			    \le
			    \prn{\ti[i+1] - \ti} \EEBrk{\indEvent{E_i} \Jof{K_{\ti}}} 
			    + 
			    4 \costMatUpper \kappa^6 \xBreak.
			\]
		\end{lemma}
		
		We further relate the lemma's bound to the cost of the optimal policy using \cref{lemma:recht}. This gets us
	    \begin{align*}
		   \prn{\ti[i+1] - \ti}\EEBrk{\indEvent{E_i} \Jof{K_{\ti}}} 
		   &\le
		   \prn{\ti[i+1] - \ti}\prn{\Jstar
		   +
		   \rechtConst \rechtEps^2 4^{-i}} \\
		   &\le
		   \prn{\ti[i+1] - \ti} \Jstar
		   +
		   3 \rechtConst \rechtEps^2 \tBase.
		\end{align*}
		Next, summing over $i$, noticing that $\sum_{i=0}^{\nT} \ti[i+1] - \ti \le T$, and simplifying the arguments yields
		\begin{equation*}
		    R_1
		    \le
		    T \cdot \Jstar
		    +
		    \nT \prn{6 \rechtConst \rechtEps^2 \tBase + 8 \costMatUpper \kappa^6 \xBreak}.
		\end{equation*}
	
		Moving to $R_2$, let $\tAbort$ be the time when the algorithm decides to abort, formally,
		$$
		\tAbort = \min\brc[1]{t \ge \tBase \bigm| \norm{x_t}^2 > \xBreak \text{ or } \norm{K_t} > \kappa},
		$$
		where we treat $\min{\emptyset} = T+1$.
		Then we have the following bound on $R_2$.
		\begin{equation*}
		    R_2
		    \le
		    {\EEBrk{\indEvent{\goodEventA^c}\sum_{t=\tBase}^{\tAbort - 1} c_t}}
		    +
		    {\EEBrk{\sum_{t=\tAbort}^{T} c_t}}.
		\end{equation*}
		Now, the state and control policy before $\tAbort$ are bounded by $\xBreak$ and $\kappa$ respectively hence $c_t \le 2 \costMatUpper \kappa^2 \xBreak$. Further recalling that $\PP{\goodEventA^c} \le T^{-2}$ bounds the first term. After $\tAbort$ the stable controller $K_0$ is played for the remaining period. This ensures that the state will not keep growing however some care is required as the state at $\tAbort$, $x_{\tAbort}$, is not bounded. The above is made formal in the following lemma.
		\begin{lemma} \label{lemma:R2AlgA}
		    $
		    R_2
		    \le
		    	\Jof{K_0}
		    	+
		    	2 \costMatUpper \kappa^2 \xBreak
		    + o(1)
		   .
		   $
		\end{lemma}

		Last, for $R_3$, the strongly stable controller $K_0$ is played throughout warm-up. Unlike $R_2$, here the initial state $x_1 = 0$ is clearly bounded and thus it is not difficult to show that $R_3$ scales linearly with the warm-up duration $\tBase$. 
		Since the latter behaves as $O\prn{\log T}$, the desired result is obtained. This is made formal in the following lemma.
		\begin{lemma} \label{lemma:R3AlgA}
		$
			R_3
			\le
			\tBase \Jof{K_0}
		$
		.
		\end{lemma}
		The final bound now follows by combining the bounds of $R_1,R_2$, and $R_3$ and from $\nT$, $\xBreak$, $\tBase$ being $O(\log T)$.
	\end{proof}
	
	For a full proof of \cref{lemma:JiBoundA,lemma:R2AlgA,lemma:R3AlgA}, see \cref{sec:proofsA}.

\subsection{Upper Bound for Unknown $\bm\Bstar$}
\label{sec:algB}

    We move to a setting where $\Astar$ is known, $\Bstar$ is unknown, but $\Kstar \Kstar^T \succeq \KstarLowerBound I$ for some unknown constant $\KstarLowerBound > 0$. 
    We show an efficient algorithm that achieves regret at most $O\prn{\KstarLowerBound^{-2}\log^2 T}$. 
    We propose \cref{alg:B} to that end. The algorithm operates in a similar fashion to \cref{alg:A} with warm-up with $K_0$ and then greedy with fail-safe, but with two main differences:
    \begin{enumerate}[nosep]
        \item It adds artificial noise to the action during warm-up;
        \item The warm-up length is not predetermined and implicitly depends on $\KstarLowerBound$.
    \end{enumerate}
	The first change ensures that the action space is explored uniformly during warm-up, and the second ensures that exploration continues at the same rate during the main loop where noise is not added.
    The specifics of these are made clear in what follows.
    
    \begin{algorithm}
    	\caption{} \label{alg:B}
    	\begin{algorithmic}[1]
    		\State \textbf{input:} parameters $\tBase, \xBreak, \kappa, \lambda, \KstarLowerBoundGuess$, a strongly stable controller $K_0$, and the state transition matrix $\Astar$.
    		
    		\State \textbf{initialize:} $n_T = \floor{\log_4(T/\tBase)}$, $n_s = n_T + 1$, $\ti[\nT+1] = T+1$.
    		\State \textbf{set:} $\ti \gets \tBase 4^i$, $\AlgKstarLowerBound \gets \KstarLowerBoundGuess 2^{-i}$  for all $0 \le i \le n_T$
            
            \For{$t = 1, \ldots, \tBase - 1$}
                \Comment{initial warm-up}
                \State \textbf{play} $u_t \sim \gaussDist{K_0 x_t}{\noiseStd^2 I}$
            \EndFor
            
            \For{\textbf{phase} $i = 0, \ldots, n_T$}
                \Comment{adaptive warm-up}
            
                \State $B_{\tau_i} = \argmin_B \sum_{s=1}^{\tau_i - 1} \norm{\prn{x_{s+1} - \Astar x_s} - B u_s}^2 + \lambda \norm{B}_F^2$
    			\State $K_{\tau_i} = \KoptOf{\Astar}{B_{\tau_i}}$.
    			
    			\If{$K_{\tau_i}^T K_{\tau_i} \succeq 3\AlgKstarLowerBound/2$}
    			    \State \textbf{save} $n_s = i$ and \textbf{break}.
    			\EndIf
    			
    		    \For{$t = \tau_i, \ldots, \tau_{i+1} - 1$}
    		        \State \textbf{play} $u_t \sim \gaussDist{K_0 x_t}{\noiseStd^2 I}$
    			\EndFor
    		\EndFor
    		
            \For{\textbf{phase} $i = n_s, \ldots, n_T$}
                \Comment{main loop}
                
                \State $B_{\tau_i} = \argmin_B \sum_{s=1}^{\tau_i - 1} \norm{\prn{x_{s+1} - \Astar x_s} - B u_s}^2 + \lambda \norm{B}_F^2$
    			\State $K_{\tau_i} = \KoptOf{\Astar}{B_{\tau_i}}$.
    			
    		    \For{$t = \tau_i, \ldots, \tau_{i+1} - 1$}
    			    \If{$\norm{x_t}^2 > \xBreak$ \textbf{or} $\norm{K_{\tau_i}} > \kappa$}
    				    \Comment{fail, abort}
    				    \State \textbf{abort} and play $K_0$ forever.
    			    \EndIf
    			    
    			    \State \textbf{play} $u_t = K_{\tau_i} x_t$.
    			\EndFor
    		\EndFor
    	\end{algorithmic}
    \end{algorithm}

    We now give a quantified restatement of  \cref{thm:unknownBregret}.
    
    \begin{theorem*}[\cref{thm:unknownBregret} restated] %
    	Suppose \cref{alg:B} is run with parameters
    	\begin{gather*}
    	    \kappa_0
            =
            \sqrt{\frac{\startCostBound}{\costMatLower \noiseStd^2}},
            \;\;
            \kappa
            =
            \sqrt{\frac{\optCostBound + \rechtEps^2 \rechtConst}{\costMatLower \noiseStd^2}},
            \;\;
            \tBase
            =
            \ceil{
            	\frac{80 k \lambda \prn{1 + \systemBound^2}}{\noiseStd^2 \rechtEps^2}
            	}
            ,
        	\\
            \xBreak
            =
            135 d \kappa^2 \noiseStd^2 \max\brc{\prn{1 + \systemBound}^2\kappa_0^6, 4 \kappa^6} \log\prn{4 T},
            \;\;
            \lambda
            =
            \kappa^2 \xBreak,
            \;\;
            \KstarLowerBoundGuess
        	=
        	4 \kappa \rechtConst \rechtEps
            .
    	\end{gather*}
    	Then for
    	$
    	T
    	\ge
    	\poly{\costMatLower^{-1}, \costMatUpper, \systemBound, \optCostBound, \startCostBound, d, k, \KstarLowerBound^{-1}}
    	$
    	we have
    	$
    	\EEBrk{\regret}
    	\le
    	\poly{\costMatLower^{-1}, \costMatUpper, \systemBound, \optCostBound, \startCostBound, d, k, \KstarLowerBound^{-1}}
    	\log^2 T
    	$.
    \end{theorem*}
    We provide the main ideas required to prove \cref{thm:unknownBregret}. 
    As in \cref{alg:A}, we first quantify the high probability event under which the regret of the algorithm is small.
    Let us first consider the parameter estimation error during warm-up, which is bounded by the following lemma.
    \begin{lemma}
    	\label{lemma:warmupParamEstB}
    	With probability at least $1 - T^{-2}$, it holds that $\norm{\dEstB[\ti]} \leq \rechtEps 2^{-i} \text{ for all } 0 \le i \le \nStart$.
    \end{lemma}
    
    Here we only give a sketch of the proof; for the full technical details, see \cref{sec:proofsB}.
    
    \begin{proof}[(sketch)]
        Consider \cref{lemma:parameterEst} with $z_t = u_t$, $y_{t+1} = x_{t+1} - \Astar x_t$, $V_t = \lambda I + \sum_{s=1}^{t-1} u_t u_t^T$ and $\dEstB = B_t - \Bstar$, then with probability at least $1 - \frac14 T^{-2}$
        \begin{equation*} \label{eq:paramEstBV}
            \tr{\dEstB^T V_t \dEstB}
            \le
            4 \noiseStd d \log \prn{4 d T^2 \frac{\det \prn{V_t}}{\det \prn{V_1}}} + 2\lambda k \systemBound^2,
        \end{equation*}
        for all $t \ge 1$. Hence, bounding $V_t$ from above and below as in \cref{lemma:goodOperationA} yields the desired parameter estimation error bound.
        
        Now, during warm-up $u_t \sim \gaussDist{K_0 x_t}{\noiseStd^2 I}$ which is equivalent to having $u_t = K_0 x_t + \uNoise$ where $\uNoise \sim \gaussDist{0}{\noiseStd^2 I}$ are i.i.d.~random variables. Note that just as $w_t$ provided exploration for $x_t$, here $\uNoise$ provides exploration for $u_t$. Indeed, for the lower bound, we have
        $$
        \EEBrk{V_t}
        \succeq
        \lambda I + \sum_{s=1}^{t-1} \EEBrk{u_s u_s^T}
        \succeq 
        \lambda I + \sum_{s=1}^{t-1} \EEBrk{\uNoise[s] \uNoise[s]^T}
        \succeq
        t \noiseStd^2 I
        ,
        $$
        and thus a measure concentration argument yields the desired high probability lower bound.
        For the upper bound, notice that
        $$
        \norm{V_t}
        \le
        \lambda + \sum_{s=1}^{t-1} \norm{u_s}^2
        \le
        \lambda + 2 \sum_{s=1}^{t-1} \prn[!]{\norm{K_0}^2\norm{x_s}^2 + \norm{\uNoise[s]}^2}
        ,
        $$
        and so the strong-stability of $K_0$ together with a high probability bound on the system and artificial noises yields the desired upper bound on $V_t$. Combining both upper and lower bounds concludes the proof.
    \end{proof}
	While the estimation rate during warm-up is desirable, adding constant magnitude noise to the action incurs regret that is linear in the warm-up length, even if $K_0 = \Kstar$, and as such we avoid this strategy during the main loop. 
	Nonetheless, the following lemma shows that the estimation rate continues into the main loop albeit with slightly different constants.
	
	\begin{lemma}
		\label{lemma:goodOperationB}
		Let $\gamma = 1 \big/ 2\kappa^2$. With probability at least $1 - T^{-2}$,
		\begin{enumerate}[nosep,label=(\roman*)]
			\item $K_{\ti}$ is $(\kappa,\gamma)-$strongly stable, $~\forall~ \nStart \le i \le \nT$;
			\item $\norm{x_t}^2 \le \xBreak$, $~\forall~ 1 \le t \le T$;
			\item $\norm{\dEstB[\ti]} \le \rechtEps \min\brc{2^{-\nStart}, 2 \KstarLowerBound^{-1 / 2} 2^{-i}}$, $~\forall~ \nStart < i \le \nT$.
		\end{enumerate}
	\end{lemma}
	
	We proceed with a proof sketch and defer details to \cref{sec:proofsB}.

	\begin{proof}[(sketch)]	
		The proof follows inductively by similar arguments to those of \cref{lemma:goodOperationA}, yet with the caveat that the lower bound on $V_t$ may not hold when the controller is rank deficient.
		
		To see this, recall that the algorithm plays $u_t = K_{\ti} x_t$ during the main loop as long as the abort state is not triggered, so we have
	    $$
	    \EEBrk{u_t u_t^T ~\big|~ K_{\ti}}
	    =
	    K_{\ti} \EEBrk{x_t x_t^T ~\big|~ K_t} K_{\ti}^T
	    \succeq
	    \noiseStd^2 K_{\ti} K_{\ti}^T
	    .
	    $$
	    This means that transforming the exploration of states $x_t$, provided for by the system noise $w_t$, into exploration of actions $u_t$ depends on the controller $K_{\ti}$ being strictly non-degenerate. We show that with high probability, $K_{\ti} K_{\ti}^T \succeq (\KstarLowerBound/2) I$ thus ensuring the exploration and the parameter estimation rate.
	    
	    First, suppose that the learner had knowledge of $\KstarLowerBound$ and recall that $\KstarLowerBoundGuess = 4 \kappa \rechtConst \rechtEps$. Taking
	    $
	    \nStart
	    \ge
	    \max\brc[1]{0, \log_2 (\KstarLowerBoundGuess/\KstarLowerBound)},
	    $
	    \cref{lemma:warmupParamEstB} implies that
	    $
	    \norm{\dEstB[{\ti[\nStart]}]} \le \min\brc[1]{\rechtEps, \tfrac{\KstarLowerBound}{4 \kappa \rechtConst}\!}
	    $
	    and applying \cref{lemma:recht} we get that
	    $
	    	\norm{K_{\ti[\nStart]} - \Kstar}
	    	\le
	    	\KstarLowerBound \big/ 4\kappa
	    	.
    	$
    	Further assuming that $\norm{K_{\ti[\nStart]}} \le \kappa$, which is ensured by strong-stability, simple algebra yields that $K_{\ti[\nStart]} K_{\ti[\nStart]}^T \succeq (\KstarLowerBound/2) I$.
    	
    	Now, when $\KstarLowerBound$ is unknown, we show that the break condition of the warm-up loop ensures that with high probability
    	\begin{equation}
    	\label{eq:warmupLengthBoundB}
		   	\max\brc{0, \log_2 \frac{\KstarLowerBoundGuess}{\KstarLowerBound}}
		   	\le
		   	\nStart
		   	\le
		   	2
		   	+
		   	\max\brc{0, \log_2 \frac{\KstarLowerBoundGuess}{\KstarLowerBound}}
		   	,
    	\end{equation}
    	a proof of which may be found in \cref{sec:proofsB}.
		The lower bound on $\nStart$ ensures the desired non-degeneracy of $K_{\ti[\nStart]}$, and proceeding by induction, the same follows for subsequent controllers. We note that the purpose of the upper bound on $\nStart$ is to ensure that the warm-up is not so long as to incur more than $O\prn{\KstarLowerBound^{-2}\log^2 T}$ regret.
    \end{proof}

    Proceeding from \cref{lemma:goodOperationB}, we obtain a regret decomposition similar to that of \cref{alg:A} with an added dependence on the random number of warm-up phases $\nStart$. While this randomness introduces some additional technical challenges, the proof ideas remain largely the same. For the full proof of \cref{thm:unknownBregret}, see \cref{sec:proofsB}.

\subsection{Lower Bound for Degenerate $\bm{\Kstar}$}

In this section we prove an $\Omega(\sqrt{T})$ lower bound for systems with a (nearly) degenerate optimal policy, stated in \cref{thm:lb-main}.
By Yao's principle, to establish the theorem it is enough to demonstrate a randomized construction of an LQR system such that the expected regret of any deterministic learning algorithm is large.

Fix $d=k=1$ and consider the system
\begin{equation} \label{eq:lb-system}
\begin{aligned}
    x_{t+1} &= a x_t + b u_t + w_t ~;
    \\
    c_{t\phantom{+1}} &= x_t^2 + u_t^2
    .
\end{aligned}
\end{equation}
Here, $w_t \sim \mathcal{N}(0,\sigma^2)$ are i.i.d.~Gaussian random variables, $a = 1/\sqrt{5}$ and $b = \chi \sqrt{\epsilon}$ where $\chi$ is a Rademacher random variable (drawn initially) and $\epsilon>0$ is a parameter whose value will be chosen later. For simplicity, we assume that $x_1 = 0$.
Notice that for this system, we have the bounds $\costMatUpper=\costMatLower=1$, $\systemBound=1$ and, as we will see below, the optimal cost of the system is bounded by $\optCostBound=2\sigma^2$.
Further, note that the system is controllable and $k_0=0$ is a stabilizing policy.
Our goal is to lower bound the regret, given by
\begin{align*}
    R_T 
    = 
    \sum_{t=1}^T \prn[!]{x_t^2 + u_t^2 - J(k_\star)}
    .
\end{align*}

\cref{thm:lb-main} follows directly from the following:

\begin{theorem} \label{thm:lowerb}
Assume that $T \ge 12000$ and set $\epsilon = T^{-1/2} \big/ 4$. 
Then the expected regret of any deterministic learning algorithm on on the system in \cref{eq:lb-system} satisfies
\[
    \EE{[R_T]}
    \ge 
    \frac{1}{3100} \sigma^2 \sqrt{T} - 4 \sigma^2.
\]
Here, the expectation is taken with respect to both the stochastic noise terms as well as the random variable $\chi$.
\end{theorem}

For the proof, we use the following notation.
We use $k_\star$ to denote the optimal policy for the system, which (recalling \cref{eq:Kopt,eq:riccati}) is given by 
\begin{align*}
    k_\star 
    = 
    -\frac{ab p_\star}{1 + b^2 p_\star}
    ,
\end{align*}
where $p_\star > 0$ is a positive solution to the Riccati equation
\begin{equation*}
	p_\star 
	= 
	1 + a^2 p_\star - \frac{a^2 b^2 p_\star^2}{1 + b^2 p_\star}
	=
	1 + \frac{a^2 p_\star}{1 + b^2 p_\star}
	.
\end{equation*}
Observe that for our choice of $\epsilon \le 1 / 400$ we have that $\abs{b} \leq 1/20$, and so
\begin{equation} \label{eq:lb-kp-bounds}
\begin{aligned}
	1 \leq p_\star \leq 1/(1-a^2) = 5 / 4, \\
	0.99 \sqrt{\epsilon / 5} \le \abs{k_{\star}} \le \sqrt{\epsilon / 3}.
\end{aligned}
\end{equation}
In particular, this means that the cost of the optimal policy is at most $\sigma^2 p_\star \leq 2\sigma^2$.
Further, the sign of $k_\star$ is solely determined by the sign of $\chi$.

Now, fix any deterministic learning algorithm. 
Let $x^{(t)} = (x_1,\ldots,x_t)$ denote the trajectory generated by the learning algorithm up to and including time step $t$. 
Denote by $\Pr_{+}$ and $\Pr_{-}$ the probability laws with respect to the trajectory generated conditioned on $\chi = 1$ and $\chi = -1$ respectively.

First, we lower bound the expected regret in terms of the cumulative magnitude of the algorithm's actions $u_t$.
The proof first relates the regret to the overall deviation of $u_t$ from the actions of the optimal policy $k_\star$ by using the fact that the action played by $k_\star$ at any state minimizes the Q-function of the system. Since the actions of $k_\star$ are small in expectation, the latter quantity can be in turn related to the total magnitude of the $u_t$. 

\begin{lemma} \label{lem:lb1}
Suppose $\epsilon \leq 1/400$.
The expected regret is lower bounded as
\begin{align*}
    \EEBrk[0]{R_T}
    \geq
    0.99 \, \EEBrk[3]{ \sum_{t=1}^T (u_t - k_\star x_t)^2 } - 4 \sigma^2
    ,
\end{align*}
and consequently,
\begin{align*}
    \EE{[ R_T ]} 
    \geq
    \frac13 \EEBrk[3]{\sum_{t=1}^T u_t^2} - \frac{5}{6}\sigma^2 k_\star^2 T - 4\sigma^2
    .
\end{align*}
\end{lemma}

Note that for the last bound to be meaningful, $k_\star$ indeed has to be very small so that the additive term that scales with $k_\star^2 T$ does not dominate the right hand side.
The proofs of this as well as subsequent lemmas are deferred to \cref{sec:lb-proofs}.

Next, by standard information theoretic arguments, we obtain an upper bound on the statistical distance between the probability laws of $x^{(T)}$ under $\Pr_{+}$ and $\Pr_{-}$, that scales with the total magnitude of the actions $u_t$. 

\begin{lemma} \label{lem:tv-distance1}
For the trajectory $x^{(T)}$, it holds that
\begin{align*}
    \TV{\Pr_{+}[x^{(T)}]}{\Pr_{-}[x^{(T)}]}
    \le
    \frac{\sqrt{\epsilon}}{\sigma} \sqrt{\EE{\Bigg[\sum_{t=1}^T u_{t}^2 \Bigg]}}
    ~.
\end{align*}
\end{lemma}

Our final lemma shows that most of the states visited by the algorithm have a non-trivial (constant) magnitude.
This is a straightforward consequence of the added Gaussian noise at each time step.

\begin{lemma} \label{lem:many-large-x}
Assume that $T \ge 12000$.
With probability $\geq \frac78$, at least $\frac23 T$ of the states $x_1,\ldots,x_T$ satisfy $|x_t| \ge 2\sigma/5$. 
\end{lemma}

We are now ready to prove the main result of this section.

\begin{proof}[of \cref{thm:lowerb}]
Notice that if $\EE{[ \sum_{t=1}^T u_t^2 ]} > \frac14 \sigma^2 \sqrt{T}$, then the desired lower bound is directly implied by the second inequality in \cref{lem:lb1}, as $k_\star^2 \leq \epsilon / 3 =  T^{-1/2} / 12$, so
$
    \EE{[ R_T ]} 
    \geq
    \tfrac{1}{100} \sigma^2 \sqrt{T} - 4\sigma^2
    .
$
We henceforth assume that $\EE{[ \sum_{t=1}^T u_t^2 ]} \leq \frac14 \sigma^2 \sqrt{T}$.
Plugging this into the bound of \cref{lem:tv-distance1} for the total variation distance between $\Pr_{+}$ and $\Pr_{-}$, and using our choice $\epsilon = T^{-1/2} / 4$, we obtain that
\[
    \TV{\Pr_{+}[x^{(T)}]}{\Pr_{-}[x^{(T)}]} 
    \leq
    \sqrt{\frac{\epsilon}{\sigma^2} \cdot \frac{\sigma^2}{4} \sqrt{T}}
    =
    \frac{1}{4}
    .
\]
Now, let $N_T$ denote the number of time steps in which $u_t k_\star x_t \le 0$, i.e., the number of times in which the learner has guessed the sign of $\chi$ incorrectly. 
We claim that
$ 
    \Pr[N_T \geq T/2]
    \geq 
    3/8
    .
$
To see this, denote by $N'_T$ the number of time steps $t$ in which $u_t x_t \le 0$.
Using the fact that $N'_T$ is a deterministic function of the trajectory $x^{(T)}$ together with the bound on the total variation gives
\begin{align*}
    \abs{ \Pr_{+}[N'_T \ge T/2] - \Pr_{-}[N'_T \ge T/2] }
    \leq
    \TV{\Pr_{+}[x^{(T)}]}{\Pr_{-}[x^{(T)}]}
    \leq
    \frac{1}{4}
    .
\end{align*}
Now, recall that the sign of $k_\star$ is determined by that of $\chi$. 
Thus, $\Pr_{-}[N_T \ge T/2] = \Pr_{-}[N'_T < T/2]$ and $\Pr_{+}[N_T \ge T/2] = \Pr_{+}[N'_T \ge T/2]$ from which
\begin{align} \label{eq:lb-NT}
    \Pr[N_T \geq T/2]
    &=
    \tfrac12 \Pr_{+}[N_T \ge T/2] + \tfrac12 \Pr_{-}[N_T \ge T/2]
    \nonumber\\
    &=
    \tfrac12 (1 + \Pr_{+}[N'_T \ge T/2] - \Pr_{-}[N'_T \ge T/2])
    \nonumber\\
    &\geq
    3/8
    .
\end{align}
On the other hand, \cref{lem:many-large-x} implies that with probability at least $7/8$, no less than $2T/3$ of the states $x_1,\ldots,x_T$ satisfy $|x_t| > 2 \sigma / 5$. 
Then by a union bound, with probability at least $1/4$, at least $T/6$ instances of $x_1,\ldots,x_T$ satisfy $|x_t| \ge 2 \sigma / 5$ and $u_t k_\star x_t \le 0$. 
For these instances, we have
$$
    (u_t - k_\star x_t)^2
    \geq
    k_\star^2 x_t^2
    \geq
    0.99^2\frac{4}{125} \epsilon \sigma^2,
$$
where we have bounded $k_{\star}$ as in \cref{eq:lb-kp-bounds}.
Hence, we can lower bound the regret using the first inequality in \cref{lem:lb1} as follows:
\begin{align*}
    \EE{[R_T]}
    &\geq
    0.99 \cdot \EE{\Bigg[\sum_{t=1}^T (u_t - k_\star x_t)^2\Bigg]} - 4 \sigma^2
    \\
    &\geq
    0.99^3 \cdot \frac{1}{4} \cdot \frac{T}{6} \cdot \frac{4}{125}\epsilon \sigma^2 - 4 \sigma^2
    \\
    &\geq
    \frac{1}{3100} \sigma^2 \sqrt{T} - 4 \sigma^2,
\end{align*}
where the last transition used our choice of $\epsilon$.
\end{proof}

\subsection*{Acknowledgements}

We thank Yishay Mansour for numerous helpful discussions.
This work was partially supported by the Israeli Science Foundation (ISF) grant 2549/19 and by the Yandex Initiative in Machine
Learning.

\bibliographystyle{plainnat}
\bibliography{bibliography}

\clearpage
\appendix
\onecolumn

\section{\cref{alg:A} Proofs} \label{sec:proofsA}

\subsection{The Good Event}
We begin with an explicit statement of the probabilistic events that comprise $\goodEventA$. Recall that
\begin{equation*}
A_t = \argmin_A \sum_{s=1}^{t - 1} \norm{x_{s+1} - \Bstar u_s - A x_s}^2 + \lambda \norm{A}_F^2,
\end{equation*}
and denote $\dEstA = A_t - \Astar$, $\Vx{t} = \lambda I + \sum_{s=1}^{t-1} x_t x_t^T$. Now, define the following events
\begin{align}
\label{eq:olsEventA}
\olsEventA
&=
\brc{\tr{\dEstA^T \Vx{t} \dEstA}
	\le
	4 \noiseStd^2 d \log \prn{3T^3 \frac{\det \prn{\Vx{t}}}{\det \prn{\Vx{1}}}} + 2\lambda d \systemBound^2, \text{ for all } t \ge 1}, \\
\label{eq:xNoiseExploreEventA}
\xNoiseExploreEventA
&=
\brc{\sum_{t=1}^{\ti - 1} x_t x_t^T \succeq \frac{\prn{\ti - 1} \noiseStd^2}{40}I, \text{ for all } 0 \le i \le \nT}, \\
\label{eq:xNoiseBoundEventA}
\xNoiseBoundEventA
&=
\brc{\max_{1 \le t \le T} \norm{w_t} \le \noiseStd \sqrt{15 d \log 3 T}},
\end{align}
Then we have the following lemma.
\begin{lemma} \label{lemma:goodEventA}
	Let $\goodEventA = \olsEventA \cap \xNoiseExploreEventA \cap \xNoiseBoundEventA$, and suppose that $T \ge 600 d \log 36 T$. Then we have that $\PP{\goodEventA} \ge 1 - T^{-2}$.
\end{lemma}

\begin{proof}
	First, we describe the parameter estimation error in terms of \cref{lemma:parameterEst}. To that end, let $z_t = x_t$, $y_{t+1} = x_{t+1} - \Bstar u_t$, $\Vx{t} = \lambda I + \sum_{s=1}^{t-1} x_t x_t^T$, and $\dEstA = A_t - \Astar$ Indeed, we have $y_{t+1} = \Astar x_t + w_t$, $w_t \sim \gaussDist{0}{\noiseStd^2 I}$, and $\norm{\Astar}_F^2 \le d \norm{\Astar}^2 \le d \systemBound^2$ and so taking \cref{lemma:parameterEst} with $\delta = \frac{1}{3}T^{-2}$, recalling that $T \ge d$, and simplifying, we get that $\PP{\olsEventA} \ge 1 - \frac{1}{3}T^{-2}$.
	
	Next, for $\xNoiseExploreEventA$, we apply \cref{lemma:VtLowerBound} to the sequence $x_t$ with the filtration $\mathcal{F}_t = \sigma\prn{x_1, u_1, \ldots, x_t, u_t}$. Notice that given $x_{t-1}, u_{t-1}$ we have $x_t \sim \gaussDist{\Astar x_{t-1} + \Bstar u_{t-1}}{\noiseStd^2 I}$ and hence we also get
	\begin{equation*}
		\EEBrk{x_t x_t^T ~\big|~ \mathcal{F}_{t-1}}
		\succeq
		\prn{\Astar x_{t-1} + \Bstar u_{t-1}} \prn{\Astar x_{t-1} + \Bstar u_{t-1}}^T
		+ \noiseStd^2 I
		\succeq
		\noiseStd^2 I.
	\end{equation*}
	Finally, our choice of $\tBase$ ensures the minimal sum size assumption. We thus apply \cref{lemma:VtLowerBound} $\nT + 1$ times with $\delta = \frac{1}{3} T^{-3}$ and apply a union bound. Since $\nT + 1 \le T$ we conclude that $\PP{\xNoiseExploreEventA} \ge 1 - \frac{1}{3} T^{-2}$.
	
	Finally, for $\xNoiseBoundEventA$ we apply \cref{lemma:noiseBound} with $\delta = \frac{1}{3} T^{-2}$ to get $\PP{\xNoiseBoundEventA} \ge 1 - \frac13 T^{-2}$.
	The final result is obtained by taking a union bound over the three events.
\end{proof}

\subsection{Proof of \cref{lemma:goodOperationA}} \label{sec:proofOfLemmaGoodOperationA}
	We first need the following two lemmas.
	\begin{lemma}[Bounded warm-up] \label{lemma:boundedWarmupA}
		On $\goodEventA$ we have that
		$
		\norm{x_t}
		\le
		\noiseStd \kappa_0^3 \sqrt{60 d \log 3T}
		\le
		\sqrt{\xBreak}
		,
		$
		for all $1 \le t \le \tBase$.
	\end{lemma}
	\begin{proof}
		First, by \cref{lemma:costToStability}, $\Jof{K_0} \le \startCostBound$ implies that $K_0$ is $(\kappa_0, \gamma_0)-$strongly stable with $\gamma_0^{-1} = 2 \kappa_0^2$.
		So, applying \cref{lemma:singleControlBound} with $x_1 = 0$ we get that for all $1 \le t \le \tBase$
		\begin{equation*}
		\norm{x_t}
		\le
		2 \kappa_0^3 \max_{1 \le t \le T}\norm{w_t},
		\end{equation*} 
		and applying the noise bound in \cref{eq:xNoiseBoundEventA} we obtain the desired result.
	\end{proof}
	
	\begin{lemma}[Conditional parameter estimation] \label{lemma:conditionalParamEstA}
		On $\goodEventA$  fix some $i$ such that $0 \le i \le \nT$ and suppose that $\norm{x_t}^2 \le \xBreak$ for all $1 \le t \le \ti$. Then we have that $\norm{\dEstA[\ti]} \le \rechtEps 2^{-i}$.
	\end{lemma}
	\begin{proof}
		First, on $\goodEventA$ by \cref{eq:xNoiseExploreEventA} we have that
		\begin{equation*}
		\Vx{\ti}
		=
		\lambda I
		+
		\sum_{t=1}^{\ti - 1} x_t x_t^T
		\succeq
		\prn{
			\lambda
			+
			\frac{\prn{\ti - 1} \noiseStd^2}{40}
		} I
		\succeq
		\frac{\ti \noiseStd^2}{40} I,
		\end{equation*}
		and so we conclude that
		\begin{equation*}
		\tr{\dEstA[\ti]^T \Vx{\ti} \dEstA[\ti]}
		\ge
		\tr{\dEstA[\ti]^T \dEstA[\ti]} \frac{\ti \noiseStd^2}{40}
		\ge
		\norm{\dEstA[\ti]}^2 \frac{\ti \noiseStd^2}{40}.
		\end{equation*}
		Rearranging and applying \cref{eq:olsEventA} we obtain
		\begin{equation*}
		\norm{\dEstA[\ti]}^2
		\le
		\frac{1}{\ti}
		\prn{
			160 d \log \prn{3 T^3 \frac{\det \prn{\Vx{\ti}}}{\det \prn{\Vx{1}}}}
			+
			80\frac{\lambda d \systemBound^2}{\noiseStd^2}
		}
		.
		\end{equation*}
		Now, since we assumed $\norm{x_t}^2 \le \xBreak = \lambda$, we can apply \cref{lemma:logDetBound} to conclude that
		\begin{equation*}
		\log \frac{\det\prn{\Vx{\ti}}}{\det\prn{\Vx{1}}}
		\le
		d \log T,
		\end{equation*}
		and plugging this into the above we get that
		\begin{align*}
		\norm{\dEstA[\ti]}^2
		\le
		\frac{1}{\ti}
		\prn{
			640 d^2 \log \prn{3 T}
			+
			80\frac{\lambda d \systemBound^2}{\noiseStd^2}
		}
		\le
		\frac{1}{\ti}
		\frac{80\lambda d \prn{1 + \systemBound^2}}{\noiseStd^2}
		\le
		\frac{\rechtEps^2 \tBase}{\ti}
		\le
		\rechtEps^2 4^{-i},
		\end{align*}
		where all transitions are due to our choice of parameters.
	\end{proof}

	\begin{proof}[of \cref{lemma:goodOperationA}]
	
		First recall that by \cref{lemma:goodController}, if $\norm{\dEstA} \le \rechtEps$ then $K_t$ is $(\kappa,\gamma)-$strongly stable.
		We now show by induction on $n$ that for all $0 \le i \le n,$ $K_{\ti}$ is $(\kappa,\gamma)-$strongly stable. Note that $0 \le n \le \nT$.
		
		For the base case, $n=0$, \cref{lemma:boundedWarmupA} shows that $\norm{x_t}^2 \le \xBreak$ for all $1 \le t \le \tBase$, which in turn satisfies \cref{lemma:conditionalParamEstA}, i.e., $\norm{\dEstA[\tBase]} \le \rechtEps$ and so the required strong stability of $K_{\tBase}$ is obtained.
		
		Now, suppose the induction holds up to $n-1$ and we show for $n$. By the strong stability of the controllers up to time $\ti[n] - 1$, and since $\tBase \ge \frac{\log \kappa}{\gamma}$, we can apply \cref{lemma:multiControlBound} to conclude that
		\begin{equation*}
			\norm{x_t}
			\le
			3\kappa
			\max\brc{
				\frac{\norm{x_{\tBase}}}{2},
				\frac{\kappa}{\gamma} \max_{1 \le t \le T}\norm{w_t}},
			\qquad \text{for all } \tBase \le t \le \ti.
		\end{equation*}
		recalling that $\gamma^{-1} = 2 \kappa^2$, bounding the noise with \cref{eq:xNoiseBoundEventA}, and bounding $\norm{x_{\tBase}}$ by \cref{lemma:boundedWarmupA} we get that
		\begin{align*}
			\norm{x_t}
			\le
			3\kappa
			\max\brc{
				\frac{\noiseStd \kappa_0^3 \sqrt{60 d \log 3T}}{2},
				2\kappa^3 \noiseStd \sqrt{15 d \log 3 T}},
			\le
			\noiseStd \kappa
			\max\brc{\kappa_0^3, 2 \kappa^3}
			\sqrt{135 d \log 3 T}
			=
			\sqrt{\xBreak},
		\end{align*}
		and as for the base case, we can now invoke \cref{lemma:conditionalParamEstA,lemma:goodController} to conclude the strong stability of $K_{\ti[n]}$ and finish the induction. Notice that this together with the above equation also show the algorithm does not abort.
		
		The induction proves the first part of the lemma, i.e., all controller are strongly-stable. Now, we can apply \cref{lemma:multiControlBound} once more to conclude that $\norm{x_t}^2 \le \xBreak$ for all $\tBase \le t \le T$ and together with \cref{lemma:boundedWarmupA} this concludes the second claim of the lemma.
		
		Finally, the third claim is now an immediate corollary of the \cref{lemma:conditionalParamEstA}.
	\end{proof}

\subsection{Proof of \cref{lemma:JiBoundA}} \label{sec:proofOfLemmaJiA}
Recall that
$
E_i
=
\brc{
	\norm{\dEstA[\ti]}
	\le
	\rechtEps 2^{-i}
}
$, and further denote
$
S_i
=
\brc{
	\norm{x_{\ti}} ^2
	\le
	\xBreak
}
.
$
Trivially, we have that $\goodEventA \subseteq E_i \cap S_i$.

Now, define $\tilde{x}_{\ti} = x_{\ti}$ and for $\ti < t \le \ti[i+1] - 1$
\begin{equation*}
	\tilde{x}_{t} = \prn{\Astar + \Bstar K_{\ti}} \tilde{x}_{t-1} + w_t.
\end{equation*}
Since on $\goodEventA$ the algorithm does not abort, we have that
$$
\indEvent{\goodEventA} J_i
=
\indEvent{\goodEventA} \sum_{t = \ti}^{\ti[i+1] - 1} \tilde{x}_t^T \prn{Q + K_{\ti}^T R K_{\ti}} \tilde{x}_t
\le
\indEvent{E_i \cap S_i} \sum_{t = \ti}^{\ti[i+1] - 1} \tilde{x}_t^T \prn{Q + K_{\ti}^T R K_{\ti}} \tilde{x}_t
.
$$
Noticing that $E_i$, $S_i$, and $K_{\ti}$ are completely determined by $x_{\ti}, A_{\ti}$ we use total expectation to get that
\begin{align*}
	\EEBrk{\indEvent{\goodEventA} J_i}
	\le
	\EEBrk{
		\indEvent{E_i \cap S_i}
		\EEBrk{
			\sum_{t = \ti}^{\ti[i+1] - 1} \tilde{x}_t^T \prn{Q + K_{\ti}^T R K_{\ti}} \tilde{x}_t
			~\bigg|~ x_{\ti}, A_{\ti}}}.
\end{align*}
Now, by \cref{lemma:goodController}, $E_i$ implies that $K_{\ti}$ is $(\kappa,\gamma)-$strongly stable and so we can use \cref{lemma:steadyStateCostConvergence} to get that
\begin{align*}
	\EEBrk{\indEvent{\goodEventA} J_i}
	&\le
	\prn{\ti[i+1] - \ti}\EEBrk{\indEvent{E_i}\Jof{K_{\ti}}}
	+
	\frac{2 \costMatUpper \kappa^4}{\gamma} \EEBrk{\indEvent{S_i}\norm{x_{\ti}}^2} \\
	&\le
	\prn{\ti[i+1] - \ti}\EEBrk{\indEvent{E_i}\Jof{K_{\ti}}}
	+
	4 \costMatUpper \kappa^6 \xBreak
	,
\end{align*}
where the second transition also used that $\gamma^{-1} = 2\kappa^2$ and the third used our choice of $\xBreak \ge \noiseStd^2 \kappa^4$.

\subsection{Proof of \cref{lemma:R2AlgA} ($\bm{R_2}$ upper bound)} \label{sec:proofOfLemmaR2A}
	We first need the following lemma.
	\begin{lemma}[Expected abort state] \label{lemma:abortStateBoundA}
		Suppose that $\PP{\tAbort \le T} \le T^{-2}$. Then we have that
		\begin{align*}
		\EEBrk{\norm{x_{\tAbort}}^2 \indEvent{\tAbort < T}}
		\le
		\prn{1 + 8 \systemBound^2}\prn{\kappa^2 + \kappa_0^2} \xBreak T^{-2}
		.
		\end{align*}
	\end{lemma}
	\begin{proof}
		First, by the lemmas assumption, we can apply \cref{lemma:expectedMaxNoise} to get that
		\begin{equation*}
			\EEBrk{\indEvent{\tAbort \le T} \max_{1 \le t \le T}\norm{w_t}^2}
			\le
			5 d \noiseStd^2 T^{-2} \log 3 T.
		\end{equation*}
		Now, notice that $\norm{\Astar + \Bstar K} \le 2 \systemBound \norm{K}$ and split into two cases. First, if $\tAbort > \tBase$ then by definition of $\tAbort$ we have that
		\begin{equation*}
		\norm{x_{\tAbort}}
		=
		\norm{\prn{\Astar + \Bstar K_{\tAbort - 1}} x_{\tAbort - 1} + w_{\tAbort - 1}}
		\le
		2 \systemBound \kappa \sqrt{\xBreak} + \max_{1 \le s \le T}\norm{w_t},
		\end{equation*}
		and taking expectation we get that
		\begin{align*}
		\EEBrk{\indEvent{\tBase < \tAbort \le T} \norm{x_{\tAbort}}^2}
		\le
		8 \systemBound^2 \kappa^2 \xBreak T^{-2}
		+
		5 d \noiseStd^2 T^{-2} \log 3 T
		\le
		\prn{1 + 8 \systemBound^2} \kappa^2 \xBreak T^{-2}
		.
		\end{align*}
		On the other hand if $\tAbort = \tBase$ then
		\begin{equation*}
		\norm{x_{\tAbort}}
		=
		\norm{\prn{\Astar + \Bstar K_0} x_{\tBase - 1} + w_{\tBase-1}}
		\le
		2 \systemBound \kappa_0 \norm{x_{\tBase - 1}} + \max_{1 \le t \le T}\norm{w_t}
		\le
		\prn{4 \systemBound + 1} \kappa_0^4 \max_{1 \le t \le T}\norm{w_t}
		,
		\end{equation*}
		where the last transition used \cref{lemma:singleControlBound} and $\gamma_0^{-1} = 2 \kappa_0^2$. Taking expectation we get that
		\begin{align*}
		\EEBrk{\indEvent{\tAbort = \tBase} \norm{x_{\tAbort}}^2}
		\le
		20 \prn{1 + 8 \systemBound^2} \kappa_0^8 d \noiseStd^2 T^{-2} \log 3 T
		\le
		\prn{1 + 8 \systemBound^2} \kappa_0^2 \xBreak T^{-2}
		,
		\end{align*}
		and combining both cases yields the final bound.
	\end{proof}

	\begin{proof}[of \cref{lemma:R2AlgA}]
		First, recall the decomposition of $R_2$.
		\begin{equation*}
			R_2
			\le
			\EEBrk{\indEvent{\goodEventA^c}\sum_{t=\tBase}^{\tAbort - 1} c_t}
			+
			{\EEBrk{\sum_{t=\tAbort}^{T} c_t}}.
		\end{equation*}
		For $\tBase \le t < \tAbort$ we have that $\norm{x_t}^2 \le \xBreak$ and $\norm{K_t} \le \kappa$ and so we get that
		$$
			c_t
			=
			x_t^T \prn{Q + K_t^T R K_t} x_t
			\le
			\norm{x_t}^2 \prn{\norm{Q} + \norm{R}\norm{K_t}^2}
			\le
			2 \costMatUpper \kappa^2 \xBreak
			.
		$$
		By \cref{lemma:goodOperationA} we have that $\PP{\goodEventA^c} \le T^{-2}$ and so we get that
		\begin{align} \label{eq:R21}
			\EEBrk{\indEvent{\goodEventA^c}\sum_{t=\tBase}^{\tAbort - 1} c_t}
			\le
			\EEBrk{\indEvent{\goodEventA^c}2 \costMatUpper \kappa^2 \xBreak T}
			=
			2 \costMatUpper \kappa^2 \xBreak T \PP{\goodEventA^c}
			\le
			2 \costMatUpper \kappa^2 \xBreak T^{-1},
		\end{align}
		bounding the first term of $R_2$.
		Next, for $t \ge \tAbort$ we have that $K_t = K_0$ and so we can apply \cref{lemma:steadyStateCostConvergence} to relate the expected cost of this period to that of the steady state cost of $K_0$. we get that
		\begin{align*}
		\EEBrk{\sum_{t=\tAbort}^T c_t}
		&=
		\EEBrk{\EEBrk{\sum_{t=\tAbort}^{T} x_t^T \prn{Q +K_0^T R K_0} x_t ~\Big|~ \tAbort, x_{\tAbort}}} \\
		&\le
		\EEBrk{\indEvent{\tAbort \le T}\prn{T \Jof{K_0} + \frac{2 \costMatUpper \kappa_0^4}{\gamma_0} \norm{x_{\tAbort}}^2}} \\
		&=
		T \Jof{K_0} \PP{\tAbort \le T} + 4 \costMatUpper \kappa_0^6 \EEBrk{\norm{x_{\tAbort}}^2 \indEvent{\tAbort \le T}},
		\end{align*}
		where the last transition used $\gamma_0^{-1} = 2 \kappa_0^2$.
		Now, by \cref{lemma:goodOperationA} we know that on $\goodEventA$ the algorithm does not abort. We conclude that $\brc{\tAbort \le T} \subseteq \goodEventA^c$ which in turn implies $\PP{\tAbort \le T} \le \PP{\goodEventA^c} \le T^{-2}$. We get that
		\begin{align*}
		\EEBrk{\sum_{t=\tAbort}^T c_t}
		\le
		\Jof{K_0} T^{-1} + 4 \costMatUpper \kappa_0^6 \EEBrk{\norm{x_{\tAbort}}^2 \indEvent{\tAbort \le T}},
		\end{align*}
		Finally, we use \cref{lemma:abortStateBoundA} and simplify to get that
		\begin{align*}
			R_2
			&\le
			2 \costMatUpper \kappa^2 \xBreak T^{-1}
			+
			\Jof{K_0} T^{-1}
			+
			4 \costMatUpper \kappa_0^6
				\prn{1 + 8 \systemBound^2}\prn{\kappa^2 + \kappa_0^2} \xBreak T^{-2} \\
			&=
			\prn{
				\Jof{K_0}
				+
				2 \costMatUpper \kappa^2 \xBreak
			} T^{-1}
			+
			4 \costMatUpper \kappa_0^6 \prn{1 + 8 \systemBound^2}\prn{\kappa^2 + \kappa_0^2} \xBreak T^{-2},
		\end{align*}
		as desired.
	\end{proof}

\subsection{Proof of \cref{lemma:R3AlgA}} \label{sec:proofOfLemmaR3A}
	Notice that for $t < \tBase$ we have that $K_t = K_0$. Moreover, we have that $x_1 = 0$. Applying \cref{lemma:steadyStateCostConvergence} we get that
	\begin{align*}
		R_3
		=
		\EEBrk{\sum_{t=1}^{\tBase - 1} x_t^T \prn{Q + K_0^T R K_0} x_t}
		\le
		\tBase \Jof{K_0}
		.
	\end{align*}

\section{\cref{alg:B} Proofs} \label{sec:proofsB}

\subsection{The Good Event} \label{sec:goodEventB}

	We begin by stating the probabilistic events that guarantee the ``good'' operation of the algorithm. To that end, it will be convenient to specify how the randomized actions during the warm-up stage are generated. For $t = 1, \ldots, T$ let $\uNoise \sim \gaussDist{0}{\noiseStd^2 I}$ be i.i.d.~samples generated before the algorithm starts. Define $\uVirt = K_0 x_t + \uNoise$ and if at time $t$ the algorithm chooses at random, i.e., during warm-up, then choose $u_t = \uVirt$. These virtual actions are a convenient technical tool as they do not directly depend on the action chosen by the algorithm.
	
	Now, recall that
	\begin{equation*}
	B_t = \argmin_B \sum_{s=1}^{t - 1} \norm{\prn{x_{s+1} - \Astar x_s} - B u_s}^2 + \lambda \norm{B}_F^2,
	\end{equation*}
	and denote $\dEstB = B_t - \Bstar$, $\Vu{t} = \lambda I + \sum_{s=1}^{t-1} u_t u_t^T$. Further recalling that $\ti = \tBase 4^i$ for $0 \le i \le \nT$ and $\ti[\nT+1] = T+1 \le \tBase 4^{\nT+1}$, we define the following events
	\begin{align}
	\label{eq:olsEventB}
	\olsEventB
	&=
	\brc{\tr{\dEstB^T \Vu{t} \dEstB}
		\le
		4 \noiseStd^2 d \log \prn{4T^3 \frac{\det \prn{\Vu{t}}}{\det \prn{\Vu{1}}}} + 2\lambda k \systemBound^2, \text{ for all } t \ge 1}, \\
	\label{eq:xNoiseExploreEventB}
	\xNoiseExploreEventB
	&=
	\brc{\sum_{t=\ti[i-1]}^{\ti - 1} x_t x_t^T \succeq \frac{ \prn{\ti - \ti[i-1]} \noiseStd^2}{40}I, \text{ for all } 1 \le i \le \nT}, \\
	\label{eq:xNoiseBoundEventB}
	\xNoiseBoundEventB
	&=
	\brc{\max_{1 \le t \le T} \norm{w_t} \le \noiseStd \sqrt{15 d \log 4 T}} \\
	\label{eq:uNoiseExploreEventB}
	\uNoiseExploreEventB
	&=
	\brc{\sum_{t=1}^{\ti - 1} \uVirt \uVirt^T \succeq \frac{(\ti - 1) \noiseStd^2}{40}I, \text{ for all } 0 \le i \le \nT}, \\
	\label{eq:uNoiseBoundEventB}
	\uNoiseBoundEventB
	&=
	\brc{\max_{1 \le t \le T} \norm{\uNoise} \le \noiseStd \sqrt{15 d \log 4T}}.
	\end{align}
	Then we have the following lemma.
	\begin{lemma} \label{lemma:goodEventB}
		Let $\goodEventB = \olsEventB \cap \xNoiseExploreEventB \cap \xNoiseBoundEventB \cap \uNoiseExploreEventB \cap \uNoiseBoundEventB$, and suppose that $T \ge 600 d \log 48 T$. Then we have that $\PP{\goodEventB} \ge 1 - T^{-2}$.
	\end{lemma}
	
	\begin{proof}
		First, we describe the parameter estimation error in terms of \cref{lemma:parameterEst}. To that end, let $z_t = u_t$, $y_{t+1} = x_{t+1} - \Astar x_t$, $\Vu{t} = \lambda I + \sum_{s=1}^{t-1} u_t u_t^T$, and $\dEstB = B_t - \Bstar$ Indeed, we have $y_{t+1} = \Bstar x_t + w_t$, $w_t \sim \gaussDist{0}{\noiseStd^2 I}$, and $\norm{\Bstar}_F^2 \le k \norm{\Bstar}^2 \le k \systemBound^2$ and so taking \cref{lemma:parameterEst} with $\delta = \frac14 T^{-2}$, recalling that $T \ge d$, and simplifying, we get that $\PP{\olsEventB} \ge 1 - \frac14 T^{-2}$.
		
		Next, for $\xNoiseExploreEventB$, we apply \cref{lemma:VtLowerBound} to the sequence $x_t$ with the filtration $\mathcal{F}_t = \sigma\prn{x_1, u_1, \ldots, x_t, u_t}$. Notice that given $x_{t-1}, u_{t-1}$ we have $x_t \sim \gaussDist{\Astar x_{t-1} + \Bstar u_{t-1}}{\noiseStd^2 I}$ and hence we also get
		\begin{equation*}
		\EEBrk{x_t x_t^T ~\big|~ \mathcal{F}_{t-1}}
		\succeq
		\prn{\Astar x_{t-1} + \Bstar u_{t-1}} \prn{\Astar x_{t-1} + \Bstar u_{t-1}}^T
		+ \noiseStd^2 I
		\succeq
		\noiseStd^2 I.
		\end{equation*}
		Notice that our choice of $\tBase$ ensures the minimal sum size assumption. We thus apply \cref{lemma:VtLowerBound} for each $1 \le i \le \nT$ with $\delta = \frac{1}{4} T^-3$ and apply a union bound to get that $\PP{\xNoiseExploreEventB} \ge 1 - \frac14 \nT T^{-3}$. Repeating the same process for $\uVirt$ we also have that $\PP{\uNoiseExploreEventB} \ge 1 - \frac14 \nT T^{-3}$.
		
		Finally, for $\xNoiseBoundEventB, \uNoiseBoundEventB$ we apply \cref{lemma:noiseBound} with $\delta = \frac{1}{4} T^{-2}$ to get that $\PP{\xNoiseBoundEventB} \ge 1 - \frac14 T^{-2}$ and $\PP{\uNoiseBoundEventB} \ge 1 - \frac14 T^{-2}$.
		
		The final result is obtained by taking a union bound over the events and noticing that $2 \nT \le T$.
	\end{proof}

\subsection{Proof of \cref{lemma:warmupParamEstB}}
The proof is implied by the last part of the following lemma.

\begin{lemma}[\cref{alg:B} good warm-up]
	\label{lemma:goodWarmupB}
	On $\goodEventB$ we have that
	\begin{enumerate}
		\item $\norm{x_t} \le \noiseStd \kappa_0^3 \prn{1 + \systemBound} \sqrt{60 d \log 4 T}$, for all $1 \le t \le \tWarmup$;
		\item $\norm{u_t}^2 \le \lambda$, for all $1 \le t < \tWarmup$;
		\item $\Vu{\ti} \succeq \frac{\ti \noiseStd^2}{40} I$, for all $0 \le i \le \nStart$;
		\item $\norm{\dEstB[\ti]} \le \rechtEps 2^{-i}$, for all $0 \le i \le \nStart$.
	\end{enumerate}
\end{lemma}

\begin{proof}
	Recall the definition of $\uNoise$ from \cref{sec:goodEventB} and define $\tilde{w}_t = w_t + B \uNoise$. then for $t \le \tWarmup$ we have that
	\begin{equation*}
		x_t
		=
		\Astar x_{t-1} + \Bstar \uVirt[t-1] + w_{t-1}
		=
		\Astar x_{t-1} + \Bstar K_0 x_{t-1} + \underbracetxt{w_{t-1} + \Bstar \uNoise[t-1]}{\tilde{w}_{t-1}},
	\end{equation*}
	i.e., we can consider $x_t$ as a sequence generated from running the controller $K_0$ on a linear system with noise sequence $\tilde{w}_t$. We can then apply \cref{lemma:singleControlBound} to get that
	\begin{equation*}
		\norm{x_t} \le \frac{\kappa_0}{\gamma_0} \max_{1 \le s \le T} \norm{\tilde{w}_s}
		\qquad, \text{ for all } 1 \le t \le \tWarmup.
	\end{equation*}
	Now, on $\goodEventB$ we have the noise bounds in \cref{eq:uNoiseBoundEventB} and \cref{eq:xNoiseBoundEventB} and so we have that
	\begin{align*}
		\max_{1 \le s \le T} \norm{\tilde{w}_s}
		\le
		\max_{1 \le s \le T} \norm{w_s} + \norm{\Bstar} \max_{1 \le s \le T} \norm{\uNoise[s]}
		\le
		\noiseStd \prn{1 + \systemBound} \sqrt{15 d \log 4 T}.
	\end{align*}
	Combining the above and recalling that $\gamma_0^{-1} = 2 \kappa_0^2$ we conclude that
	\begin{equation*}
		\norm{x_t}
		\le
		\noiseStd \kappa_0^3 \prn{1 + \systemBound} \sqrt{60 d \log 4 T}
		\qquad, \text{ for all } 1 \le t \le \tWarmup,
	\end{equation*}
	proving the first claim of the lemma.
	Next, for $1 \le t < \tWarmup$ we have that $u_t = \uVirt = K_0 x_t + \uNoise$ and so
	\begin{align*}
		\norm{u_t}
		\le
		\kappa_0 \norm{x_t} + \norm{\uNoise}
		&\le
		\noiseStd \kappa_0^4 \prn{2 + \systemBound} \sqrt{60 d \log 4 T}
		\le
		\sqrt{\lambda},
	\end{align*}
	proving the second claim.
	Next, notice that for $0 \le i \le \nStart$ we have that $\Vu{\ti} = \lambda I + \sum_{s=1}^{\ti - 1} \uVirt[s] \uVirt[s]^T$. Since $\goodEventB$ holds, we can use the warm-up actions lower bound in \cref{eq:uNoiseExploreEventB} to get that
	\begin{equation*}
		\Vu{\ti}
		\succeq
		\prn{\lambda + \frac{(\ti - 1) \noiseStd^2}{40}} I
		\succeq
		\frac{\ti \noiseStd^2}{40} I
		\qquad,
		\text{ for all }
		0 \le i \le \nStart,
	\end{equation*}
	proving the third claim. For the final claim, we first use the lower bound on $\Vu{\ti}$ to get that
	\begin{align*}
		\norm{\dEstB[\ti]}^2
		\le
		\tr{\dEstB[\ti]^T \dEstB[\ti]}
		\le
		\frac{40}{\ti \noiseStd^2} \tr{\dEstB[\ti]^T \Vu{\ti} \dEstB[\ti]}
		.
	\end{align*}
	Next, we apply \cref{eq:olsEventB} to get that
	\begin{align*}
	\norm{\dEstB[\ti]}^2
	\le
	\frac{1}{\ti}\prn{
		160 d \log \prn{4 T^3 \frac{\det \prn{\Vu{\ti}}}{\det \prn{\Vu{1}}}}
		+
		\frac{80\lambda k \systemBound^2}{\noiseStd^2}}.
	\end{align*}
	Now, using the second claim of the lemma, we can use \cref{lemma:logDetBound} to get that
	$
		\log \frac{\det \prn{\Vu{\ti}}}{\det \prn{\Vu{1}}}
		\le
		k \log T
		,
	$
	and applying it to the above and simplifying we get that
	\begin{align*}
		\norm{\dEstB[\ti]}^2
		&\le
		\frac{1}{\ti}\prn{
			160 d k \log \prn{4 T^4}
			+
			\frac{80\lambda k \systemBound^2}{\noiseStd^2}} \\
		&\le
		\frac{1}{\ti}\prn{
			640 d k \log \prn{4 T}
			+
			\frac{80\lambda k \systemBound^2}{\noiseStd^2}} \\
		&\le
		\frac{1}{\ti}\frac{80\lambda k \prn{1 + \systemBound^2}}{\noiseStd^2}
		\le
		\frac{\rechtEps^2 \tBase}{\ti}
		=
		\rechtEps^2 4^{-i},
	\end{align*}
	thus concluding the proof.
\end{proof}

\subsection{Proof of \cref{lemma:goodOperationB}}
	The proof is broken into the following lemmas. The first two claims are concluded by putting together \cref{lemma:goodWarmupB,lemma:boundedOpB} and the third is given by \cref{lemma:parmEstB}.
	
	Before proceeding, we need the two following lemmas.
	\begin{lemma}[\cref{alg:B} warm-up length]
	\label{lemma:warmupLengthB}
	On $\goodEventB$ we have that
	$
	\max\brc{0, \log_2\frac{\KstarLowerBoundGuess}{\KstarLowerBound}}
	\le
	\nStart
	\le
	2 + \max\brc{0, \log_2\frac{\KstarLowerBoundGuess}{\KstarLowerBound}}
	$.
	\end{lemma}

	\begin{proof}
		First recall that by \cref{lemma:goodWarmupB}, we have that $\norm{\dEstB[\ti]} \le \rechtEps 2^{-i}$ for all $0 \le i \le \nStart$. Now, our choice of $\KstarLowerBoundGuess$ implies that
		$
		\rechtEps
		= 
		\frac{\KstarLowerBoundGuess}{4 \kappa \rechtConst}
		$
		and further recalling that $\AlgKstarLowerBound = \KstarLowerBoundGuess 2^{-i}$, we apply \cref{lemma:goodController} to get that
		\begin{align}
		\label{eq:KtiLowerBound}
		K_{\ti} K_{\ti}^T
		\succeq
		\Kstar \Kstar^T - \frac{\AlgKstarLowerBound}{2}I
		\qquad,
		\text{ for all } 0 \le i \le \nStart \\
		\label{eq:KstarLowerBound}
		\Kstar \Kstar^T
		\succeq
		K_{\ti} K_{\ti}^T - \frac{\AlgKstarLowerBound}{2}I
		\qquad,
		\text{ for all } 0 \le i \le \nStart
		. 
		\end{align}
		Now, suppose in contradiction that $\nStart > 0$ and $\AlgKstarLowerBound[\nStart] < \frac{\KstarLowerBound}{4}$. This means that $\AlgKstarLowerBound[\nStart-1] < \frac{\KstarLowerBound}{2}$ and so we can apply \cref{eq:KtiLowerBound} to get that
		\begin{align*}
			K_{\tWarmup - 1} K_{\tWarmup - 1}^T
			\succeq
			\prn{\KstarLowerBound - \frac{\AlgKstarLowerBound[\nStart-1]}{2}} I
			\succeq
			\prn{2 \AlgKstarLowerBound[\nStart-1] - \frac{\AlgKstarLowerBound[\nStart-1]}{2}} I
			=
			\frac{3}{2} \AlgKstarLowerBound[\nStart - 1] I,
		\end{align*}
		which contradicts the fact that $\nStart$ is the first time the warm-up break condition is satisfied. We conclude that either $\nStart = 0$ or $\AlgKstarLowerBound[\nStart] \ge \frac{\KstarLowerBound}{4}$. Plugging $\AlgKstarLowerBound[\nStart] = \KstarLowerBoundGuess 2^{-\nStart}$ the latter condition implies $\nStart \le 2 + \log_2 \frac{\KstarLowerBoundGuess}{\KstarLowerBound}$ thus giving the lemma's upper bound.
		
		Now for the lower bound, suppose in contradiction that $\AlgKstarLowerBound[\nStart] > \KstarLowerBound$ then by \cref{eq:KstarLowerBound} we get that
		\begin{equation*}
			\Kstar \Kstar^T
			\succeq
			\prn{\frac{3}{2} \AlgKstarLowerBound[\nStart] - \frac{\AlgKstarLowerBound[\nStart]}{2}} I
			\succ
			\KstarLowerBound I,
		\end{equation*}
		which contradicts the fact that $\KstarLowerBound$ is the tight lower bound on the eigenvalues of $\Kstar \Kstar^T$. We conclude that $\AlgKstarLowerBound[\nStart] \le \KstarLowerBound$ which in turn implies the desired lower bound.
	\end{proof}

	\begin{lemma}[\cref{alg:B} conditional control] \label{lemma:conditionalControlB}
		Suppose $\goodEventB$ holds and fix some $i$ such that $\nStart \le i \le \nT$. If $\norm{u_t}^2 \le \lambda$ for all $1 \le t \le \ti-1$, then $K_{\ti}$ is $(\kappa,\gamma)-$strongly stable and $K_{\ti} K_{\ti}^T \succeq \frac{\KstarLowerBound}{2} I$.
	\end{lemma}
	\begin{proof}
		If $\norm{\dEstB[\ti]} \le \min\brc{\rechtEps, \frac{\KstarLowerBound}{4 \kappa \rechtConst}}$ then \cref{lemma:goodController} immediately implies the desired result. We prove this estimation error bound thus concluding the proof.
		
		To that end, notice that for $t \ge s$ we have $\Vu{t} \succeq \Vu{s}$. Using the lower bound on $\Vu{\tWarmup}$ in \cref{lemma:goodWarmupB} we get that
		\begin{align*}
		\tr{\dEstB[\ti]^T \Vu{\ti} \dEstB[\ti]}
		\ge
		\tr{\dEstB[\ti]^T \Vu{\tWarmup} \dEstB[\ti]}
		\ge
		\norm{\dEstB[\ti]}^2 \frac{\tWarmup \noiseStd^2}{40},
		\end{align*}
		and by changing sides and applying the parameter estimation bound in \cref{eq:olsEventB} we get that
		\begin{equation} \label{eq:baseEstB}
		\norm{\dEstB[\ti]}^2
		\le
		\frac{1}{\tWarmup} \prn{
			160 d \log \prn{4 T^3 \frac{\det \prn{\Vu{\ti}}}{\det \prn{\Vu{1}}}}
			+
			\frac{80 \lambda k \systemBound^2}{\noiseStd^2}}
		\end{equation}
		Now, using the assumption on $u_t$, we can apply \cref{lemma:logDetBound} to get that $\log \frac{\det \Vu{\ti}}{\det \Vu{1}} \le k \log T$. Plugging this back into \cref{eq:baseEstB} and simplifying we get that
		\begin{equation*}
		\norm{\dEstB[\ti]}
		\le
		\rechtEps 2^{-\nStart},
		\end{equation*}
		and plugging in the lower bound on $\nStart$ from \cref{lemma:warmupLengthB} gives the desired bound on the estimation error thus concluding the proof.
	\end{proof}

	\begin{lemma}[\cref{alg:B} bounded operation] \label{lemma:boundedOpB}
		On $\goodEventB$ we have that
		\begin{enumerate}
			\item $\norm{x_t}^2 \le \xBreak$, for all $\tWarmup \le t \le T$;
			\item $K_{\ti}$ is $(\kappa,\gamma)-$strongly stable, for all $\nStart \le i \le \nT$;
			\item $K_{\ti} K_{\ti}^T \succeq \tfrac12 \KstarLowerBound I$, for all $\nStart \le i \le \nT$.
		\end{enumerate}
	\end{lemma}
	\begin{proof}
		First, recall the bounds on $x_t, u_t$ from \cref{lemma:goodWarmupB}, i.e.,
		\begin{align*}
			\norm{x_{\tWarmup}}
			&\le
			\noiseStd \kappa_0^3 \prn{1 + \systemBound} \sqrt{60 d \log 4 T}
			\le
			\sqrt{\xBreak}, \\
			\norm{u_t}^2
			&\le
			\lambda, \text{ for all } 1 \le t < \tWarmup.
		\end{align*}
		We prove by induction on $n$ where $\nStart \le n \le \nT$ that the claims of the lemma hold up to time $\ti[n]$ and phase $n$ respectively.
		
		For the base case, $n = \nStart$ the bounds above satisfy \cref{lemma:conditionalControlB} and so we conclude that $K_{\tBase}$ is strongly stable and that $K_{\tBase} K_{\tBase}^T \succeq \frac12 \KstarLowerBound I$ thus satisfying the induction base.
		
		Next, assume the induction hypothesis holds for $n-1$ and we show for $n$. By the induction hypothesis, the algorithm does not abort up to (including) time $\ti[n-1]-1$. Moreover, it means that for all $\nStart \le i \le n - 1$ the controllers $K_{\ti}$ are $(\kappa,\gamma)-$strongly stable and so we can use \cref{lemma:multiControlBound} to get that
		\begin{equation*}
		\norm{x_t}
		\le
		3\kappa
		\max\brc{
			\frac{\norm{x_{\tWarmup}}}{2},
			\frac{\kappa}{\gamma} \max_{1 \le s \le T} \norm{w_t}},
		\qquad \text{for all } \tWarmup \le t \le \ti[n],
		\end{equation*}
		and plugging in that $\gamma^{-1} = 2\kappa^2$, the bound for $\norm{x_{\tWarmup}}$ and the bound for the noise in \cref{eq:xNoiseBoundEventB} we get that
		\begin{equation*}
		\norm{x_t}
		\le
		\kappa \noiseStd
		\max\brc{
			\kappa_0^3 (1 + \systemBound),
			2\kappa^3
		}
		\sqrt{135 d \log 4 T}
		\le
		\sqrt{\xBreak},
		\qquad \text{for all } \tWarmup \le t \le \ti[n],
		\end{equation*}
		as desired for $x_t$. Notice that this ensures that the algorithm does not abort up to time $\ti[n] - 1$. So, for $\tWarmup \le t \le \ti[n] - 1$ we have that $\norm{u_t} = \norm{K_t x_t} \le \norm{K_t} \norm{x_t} \le \kappa \sqrt{\xBreak} = \sqrt{\lambda}$, and thus \cref{lemma:conditionalControlB} establishes the desired strong-stability and non-degeneracy of $K_{\ti[n]}$, finishing the induction.
		
		Finally, using the strong stability of all controllers we apply \cref{lemma:multiControlBound} a final time to obtain the bound on $x_t$ for all $\tWarmup \le t \le T$.
	\end{proof}

	\begin{lemma}[\cref{alg:B} parameter estimation] \label{lemma:parmEstB}
		On $\goodEventB$ we have that
		$\norm{\dEstB[\ti]} \le \rechtEps \min\brc{2^{-\nStart}, 2 \KstarLowerBound^{-1 / 2} 2^{-i}}$, $~\forall~ \nStart < i \le \nT$.
	\end{lemma}
	\begin{proof}
		Recall that by \cref{lemma:boundedOpB}, the algorithm does not abort on $\goodEventB$ and so for $\ti \le t \le \ti[i+1] - 1$ we have that $K_t = K_{\ti}$. This means we can decompose $\Vu{\ti}$ as
		\begin{align*}
			\Vu{\ti}
			=
			\Vu{\tWarmup}
			+
			\sum_{t=\tWarmup}^{\ti - 1} u_t u_t^T
			=
			\Vu{\tWarmup}
			+
			\sum_{j=\nStart}^{i-1} \sum_{t=\ti[j]}^{\ti[j+1] - 1} u_t u_t^T 
			=
			\Vu{\tWarmup}
			+
			\sum_{j=\nStart}^{i-1} K_{\ti[j]} \prn{\sum_{t=\ti[j]}^{\ti[j+1] - 1} x_t x_t^T} K_{\ti[j]}^T.
		\end{align*}
		Next, we lower bound $\Vu{\tWarmup}$ using \cref{lemma:goodWarmupB} and the states using \cref{eq:xNoiseExploreEventB} and get that
		\begin{align*}
			\Vu{\ti}
			\succeq
			\frac{\tWarmup \noiseStd^2}{40} I
			+
			\sum_{j=\nStart}^{i-1} \prn{\frac{\prn{\ti[j+1] - \ti[j]} \noiseStd^2}{40}} K_{\tWarmup 2^{j-1}} K_{\tWarmup 2^{j-1}}^T,
		\end{align*}
		and recalling that $K_{\ti[j]} K_{\ti[j]}^T \succeq \frac{\KstarLowerBound}{2} I$ (see \cref{lemma:boundedOpB}) we get that, assuming $i > \nStart$,
		\begin{align*}
			\Vu{\ti}
			\succeq
			\frac{\noiseStd^2}{40} \prn{
				\tWarmup
				+
				\frac{\KstarLowerBound}{2} \sum_{j=\nStart}^{i-1} (\ti[j+1] - \ti[j])
				} I
			=
			\frac{\noiseStd^2}{40} \prn{
				\tWarmup
				+
				\frac{\KstarLowerBound}{2} \prn{\ti - \tWarmup}
			} I
			\succeq
			\frac{\noiseStd^2}{40} \max\brc{\tWarmup, \frac{\KstarLowerBound}{4} \ti} I.
		\end{align*}
		Now, apply this together with the parameter estimation bound in \cref{eq:olsEventB} to get that
		\begin{align*}
			\norm{\dEstB[\ti]}^2
			&\le
			\tr{\dEstB[\ti]^T \dEstB[\ti]} \\
			&\le
			\frac{40}{\noiseStd^2 \max\brc{\tWarmup, \frac{\KstarLowerBound}{4} \ti}} \tr{\dEstB[\ti]^T \Vu{\ti} \dEstB[\ti]} \\
			&\le
			\frac{1}{\max\brc{\tWarmup, \frac{\KstarLowerBound}{4} \ti}}
			\prn{
				160 d \log \prn{4 T^3 \frac{\det \prn{\Vu{\ti}}}{\det \prn{\Vu{1}}}}
				+
				\frac{80 \lambda k \systemBound^2}{\noiseStd^2}
			}.
		\end{align*}
		Finally, from \cref{lemma:goodWarmupB} we have that $\norm{u_t}^2 \le \lambda$ for $1 \le t < \tWarmup$ and from \cref{lemma:boundedOpB} we have that $\norm{x_t}^2 \le \xBreak$ for $\tWarmup \le t \le T$ and so $\norm{u_t}^2 = \norm{K_t x_t}^2 \le \kappa^2 \xBreak = \lambda$. Combining both claims, we apply \cref{lemma:logDetBound} to get that $\log \frac{\det \Vu{\ti}}{\det \Vu{1}} \le k \log T$ and plugging this into the above equation we get
		\begin{align*}
		\norm{\dEstB[\ti]}^2
		\le
		\frac{1}{\max\brc{\tWarmup, \frac{\KstarLowerBound}{4} \ti}}
		\prn{
			640 d k \log \prn{4 T}
			+
			\frac{80 \lambda\systemBound^2}{\noiseStd^2}
		}
		\le
		\frac{\tBase \rechtEps^2}{\max\brc{\tWarmup, \frac{\KstarLowerBound}{4} \ti}}
		=
		\rechtEps^2 \min\brc{4^{-\nStart}, \frac{4}{\KstarLowerBound} 4^{-i}},
		\end{align*}
		where the second transition follows from our choice of $\tBase$.
	\end{proof}

\subsection{Proof of \cref{thm:unknownBregret}}

	As in \cref{alg:A}, denote $J_i = \sum_{t = \ti}^{\ti[i+1] - 1} x_t^T \prn{Q + K_{\ti}^T R K_{\ti}} x_t$. Recalling that warm-up lasts until phase $\nStart$,
	we have the following decomposition of the regret:
	\begin{align*}
	\EEBrk{\regret}
	=
	R_1 + R_2 + R_3 - T \Jstar,
	\end{align*}
	where 
	\begin{align*}
	R_1 = \EEBrk{\sum_{i=\nStart}^{\nT} \indEvent{\goodEventB} J_i}, \qquad
	R_2 = \EEBrk{\indEvent{\goodEventB^c}\sum_{t=\tWarmup}^{T} c_t}, \qquad
	R_3 = \EEBrk{\sum_{t=1}^{\tWarmup - 1} c_t},
	\end{align*}
	are the costs due to success, failure, and warm-up respectively. 
	The following lemmas bound each of $R_1,R_2,R_3$ thus concluding the proof. The proofs for $R_1, R_2$ remain nearly the same but are provided for completeness. The proof of $R_3$ contains a few technical challenges, introduced by the randomness of the warm-up period duration.
    
    \begin{lemma} \label{lemma:R1AlgB}
		$R_1 - T \Jstar
		\le
		\nT \prn{
			6 \rechtConst \rechtEps^2 \max\brc{1, 4 \KstarLowerBound^{-1}} \tBase
			+
			8 \costMatUpper \kappa^6 \xBreak
		}
		$.
	\end{lemma}
    
    \begin{lemma} \label{lemma:R2AlgB}
    	$R_2
    	\le
    	\prn{
    		\Jof{K_0}
    		+
    		2 \costMatUpper \kappa^2 \xBreak
    	} T^{-1}
    	+
    	4 \costMatUpper \kappa_0^6 \prn{1 + 8 \systemBound^2}\prn{\kappa^2 + \kappa_0^2} \xBreak T^{-2}
    	.
    	$
    \end{lemma}

    \begin{lemma}\label{lemma:R3AlgB}
    	$
    	R_3
    	\le
    	(1 + \systemBound^2)\prn{
    		65 \Jof{K_0} \max\brc{1, \frac{\KstarLowerBoundGuess^2}{\KstarLowerBound^2}} \tBase
    		+
    		80  \costMatUpper d \noiseStd^2 \kappa_0^{14} \log^2 3 T
    	}
   		.
    	$
    \end{lemma}

    \subsubsection{Proof of \cref{lemma:R1AlgB}}
    \begin{proof}
    	We begin by bounding $\EEBrk{\indEvent{\goodEventB} J_i ~\big|~ \nStart}$ for $\nStart \le i \le \nT$. This follows exactly as in \cref{lemma:JiBoundA} but with some changes to the events $E_i$, and thus is repeated here.
    	For $\nStart \le i \le \nT$ define the events
    	$
	    	S_i
	    	=
	    	\brc{
	    		\norm{x_{\ti}} ^2
	    		\le
	    		\xBreak
	    	}
    	$
    	and
    	\begin{align*}
	    	E_{\nStart} = \brc{\norm{\dEstB[\tWarmup]} \le \rechtEps 2^{-\nStart}}
	    	,\qquad
	    	E_i = \brc{\norm{\dEstB[\ti]} \le \rechtEps \min\brc{2^{-\nStart}, 2 \KstarLowerBound^{-1} 2^{-i}}}, ~\forall~ \nStart < i \le \nT.
    	\end{align*}    	
    	By \cref{lemma:goodOperationB}, we have that $\goodEventB \subseteq E_i \cap S_i$.
    	Now, define $\tilde{x}_{\ti} = x_{\ti}$ and for $\ti < t \le \ti[i+1] - 1$ define
    	\begin{equation*}
    	\tilde{x}_{t} = \prn{\Astar + \Bstar K_{\ti}} \tilde{x}_{t-1} + w_t.
    	\end{equation*}
    	Since on $\goodEventB$ the algorithm does not abort, we have that
    	$$
    	\indEvent{\goodEventB} J_i
    	=
    	\indEvent{\goodEventB} \sum_{t = \ti}^{\ti[i+1] - 1} \tilde{x}_t^T \prn{Q + K_{\ti}^T R K_{\ti}} \tilde{x}_t
    	\le
    	\indEvent{E_i \cap S_i} \sum_{t = \ti}^{\ti[i+1] - 1} \tilde{x}_t^T \prn{Q + K_{\ti}^T R K_{\ti}} \tilde{x}_t
    	.
    	$$
    	Noticing that $E_i$, $S_i$, and $K_{\ti}$ are completely determined by $x_{\ti}, B_{\ti}$ we use total expectation to get that
    	\begin{align*}
    	\EEBrk{\indEvent{\goodEventB} J_i ~\big|~ \nStart}
    	\le
    	\EEBrk{
    		\indEvent{E_i \cap S_i}
    		\EEBrk{
    			\sum_{t = \ti}^{\ti[i+1] - 1} \tilde{x}_t^T \prn{Q + K_{\ti}^T R K_{\ti}} \tilde{x}_t
    			~\bigg|~ x_{\ti}, B_{\ti}} ~\bigg|~ \nStart},
    	\end{align*}
    	where in the inner expectation we removed the conditioning on $\nStart$ since the $\tilde{x}_t$ are conditionally independent of $\nStart$ given $x_{\ti}$.
    	Now, by \cref{lemma:goodController}, $E_i$ implies that $K_{\ti}$ is $(\kappa,\gamma)-$strongly stable and so we can use \cref{lemma:steadyStateCostConvergence} to get that
    	\begin{align}
    	\nonumber
    	\EEBrk{\indEvent{\goodEventB} J_i}
    	&\le
    	\prn{\ti[i+1] - \ti}\EEBrk{\indEvent{E_i}\Jof{K_{\ti}} ~\big|~ \nStart}
    	+
    	\frac{2 \costMatUpper \kappa^4}{\gamma}
    		\EEBrk{\indEvent{S_i}\norm{x_{\ti}}^2 ~\big|~ \nStart} \\
    	\label{eq:JiBbound}
    	&\le
    	\prn{\ti[i+1] - \ti}\EEBrk{\indEvent{E_i}\Jof{K_{\ti}} ~\big|~ \nStart}
    	+
    	4 \costMatUpper \kappa^6 \xBreak 
    	,
    	\end{align}
    	where the second transition also used that $\gamma^{-1} = 2\kappa^2$.
    	
    	Now, by \cref{lemma:recht}, on $E_{\nStart}$ we have that
    	$
    		\Jof{K_{\tWarmup}}
    		\le
    		\Jstar
    		+
    		\rechtConst \rechtEps^2 4^{-\nStart}
    	$
    	and on $E_i$ where $\nStart < i \le \nT$, we have that
    	$
    		\Jof{K_{\ti}}
    		\le
    		\Jstar + \rechtConst \rechtEps^2 \min\brc{4^{-\nStart}, 4 \KstarLowerBound^{-1} 4^{-i}}
    		.
    	$
    	Combining both cases we conclude that
    	$$
    		\indEvent{E_i}\Jof{K_{\ti}}
    		\le
    		\Jstar + \rechtConst \rechtEps^2 \max\brc{1, 4 \KstarLowerBound^{-1}} 4^{-i}
    		\qquad,\forall~ \nStart \le i \le \nT,
    	$$
    	and plugging this back into \cref{eq:JiBbound} and recalling that $\ti[i+1] - \ti \le 3 \ti = 3\tBase 4^i$ we have that
    	\begin{align*}
    		\EEBrk{\indEvent{\goodEventB} J_i ~\big|~ \nStart}
    		\le
    		\prn{\ti[i+1] - \ti} \Jstar
    		+
    		3 \rechtConst \rechtEps^2 \max\brc{1, 4 \KstarLowerBound^{-1}} \tBase
    		+
    		4 \costMatUpper \kappa^6 \xBreak.
    	\end{align*}
    	Finally, we sum over $i$ to conclude that
    	\begin{align*}
    		R_1
    		=
    		\EEBrk{\sum_{i=\nStart}^{\nT} \EEBrk{\indEvent{\goodEventB} J_i ~\big|~ \nStart}}
    		&\le
    		\EEBrk{\sum_{i=\nStart}^{\nT}
	    		\prn{\ti[i+1] - \ti} \Jstar
	    		+
	    		3 \rechtConst \rechtEps^2 \max\brc{1, 4 \KstarLowerBound^{-1}} \tBase
	    		+
	    		4 \costMatUpper \kappa^6 \xBreak
	    	} \\
    		&\le
    		\EEBrk{\prn{\ti[\nT+1] - \tWarmup} \Jstar
    			+
    			\prn{\nT + 1 - \nStart} \prn{
	    			3 \rechtConst \rechtEps^2 \max\brc{1, 4 \KstarLowerBound^{-1}} \tBase
    				+
    				4 \costMatUpper \kappa^6 \xBreak
    				}
    		} \\
    		&\le
    		T \Jstar
   			+
   			\nT \prn{
   				6 \rechtConst \rechtEps^2 \max\brc{1, 4 \KstarLowerBound^{-1}} \tBase
   				+
   				8 \costMatUpper \kappa^6 \xBreak
   			},
    	\end{align*}
    	thus concluding the proof.
    \end{proof}

	\subsubsection{Proof of \cref{lemma:R2AlgB}}
	The proof is identical to that of \cref{lemma:R2AlgA} where the initial warm-up duration $\tBase$ is replaced with $\tWarmup$ and the uses of \cref{lemma:goodEventA,lemma:abortStateBoundA} are replaced with \cref{lemma:goodEventB,lemma:abortStateBoundB} respectively. We thus conclude by proving \cref{lemma:abortStateBoundB}.
	To that end, recall that $\tAbort$ is the time when the algorithm decides to abort, formally,
	$$
	\tAbort = \min\brc[1]{t \ge \tWarmup \bigm| \norm{x_t}^2 > \xBreak \text{ or } \norm{K_t} > \kappa},
	$$
	where we treat $\min{\emptyset} = T+1$.
	\begin{lemma}[Expected abort state] \label{lemma:abortStateBoundB}
		Suppose that $\PP{\tAbort \le T} \le T^{-2}$. Then we have that
		\begin{align*}
		\EEBrk{\norm{x_{\tAbort}}^2 \indEvent{\tAbort < T}}
		\le
		\prn{1 + 8 \systemBound^2}\prn{\kappa^2 + \kappa_0^2} \xBreak T^{-2}
		.
		\end{align*}
	\end{lemma}
	\begin{proof}
		First, by the lemmas assumption, we can apply \cref{lemma:expectedMaxNoise} to get that
		\begin{align}
		\label{eq:expectedNoiseBound}
		\EEBrk{\indEvent{\tAbort \le T} \max_{1 \le t \le T}\norm{w_t}^2}
		&\le
		5 d \noiseStd^2 T^{-2} \log 3 T, \\
		\label{eq:expectedWarmupNoiseBound}
		\EEBrk{\indEvent{\tAbort \le T} \max_{1 \le t \le T}\norm{\Bstar \uNoise + w_t}^2}
		&\le
		5 d \noiseStd^2 \prn{1 + \systemBound^2} T^{-2} \log 3 T.
		\end{align}
		Now, notice that $\norm{\Astar + \Bstar K} \le 2 \systemBound \norm{K}$ and split into two cases. First, if $\tAbort > \tWarmup$ then by definition of $\tAbort$ we have that
		\begin{equation*}
		\norm{x_{\tAbort}}
		=
		\norm{\prn{\Astar + \Bstar K_{\tAbort-1}} x_{\tAbort-1} + w_{\tAbort-1}}
		\le
		2 \systemBound \kappa \sqrt{\xBreak} + \max_{1 \le s \le T}\norm{w_t},
		\end{equation*}
		and taking expectation and applying \cref{eq:expectedNoiseBound} we get that
		\begin{align*}
		\EEBrk{\indEvent{\tWarmup < \tAbort \le T} \norm{x_{\tAbort}}^2}
		\le
		8 \systemBound^2 \kappa^2 \xBreak T^{-2}
		+
		5 d \noiseStd^2 T^{-2} \log 3 T
		\le
		\prn{1 + 8 \systemBound^2} \kappa^2 \xBreak T^{-2}
		.
		\end{align*}
		On the other hand if $\tAbort = \tWarmup$ then $u_{\tAbort - 1} = K_0 x_{\tWarmup - 1} + \uNoise[\tWarmup-1]$ and so we have that
		\begin{align*}
		\norm{x_{\tAbort}}
		&=
		\norm{\prn{\Astar + \Bstar K_0} x_{\tWarmup - 1} + \prn{\Bstar \uNoise[\tWarmup - 1] + w_{\tWarmup-1}}} \\
		&\le
		2 \systemBound \kappa_0 \norm{x_{\tWarmup - 1}} + \max_{1 \le t \le T}\norm{\Bstar \uNoise + w_t} \\
		&\le
		\prn{4 \systemBound + 1} \kappa_0^4 \max_{1 \le t \le T}\norm{\Bstar \uNoise + w_t}
		,
		\end{align*}
		where the last transition used \cref{lemma:singleControlBound} and $\gamma_0^{-1} = 2 \kappa_0^2$. Taking expectation and applying \cref{eq:expectedWarmupNoiseBound} we get that
		\begin{align*}
		\EEBrk{\indEvent{\tAbort = \tWarmup} \norm{x_{\tAbort}}^2}
		\le
		80\prn{1 + \systemBound}^2\prn[0]{1 + \systemBound^2} \kappa_0^8 d \noiseStd^2 T^{-2} \log 3 T
		\le
		\prn{1 + \systemBound^2} \kappa_0^2 \xBreak T^{-2}
		,
		\end{align*}
		and combining both cases yields the final bound.
	\end{proof}

	\subsubsection{Proof of \cref{lemma:R3AlgB}}
	
	\begin{proof}
		We begin by decomposing $R_3$. Notice that $\nStart \le \nT + 1$ and so we have that
		\begin{align*}
			R_3
			=
			\EEBrk{\sum_{t=1}^{\tBase - 1} c_t}
			+
			\EEBrk{\sum_{i=0}^{\nStart - 1}
			\sum_{t=\ti}^{\ti[i+1] - 1} c_t}
			=
			\EEBrk{\sum_{t=1}^{\tBase - 1} c_t}
			+
			\sum_{i=0}^{\nT}
			\EEBrk{\indEvent{\nStart > i}
			\sum_{t=\ti}^{\ti[i+1] - 1} c_t}.
		\end{align*}
		Now, define $\Jof{K, W}$ to be the infinite horizon cost of playing controller $K$ on the LQ system $(\Astar,\Bstar)$ whose system noise has covariance $W \in \RR[d \times d]$. In terms of our notation so far, this means that $\Jof{K} = \Jof{K, \noiseStd^2 I}$. It is well known that $\Jof{K, W} = \tr{P W}$ where $P$ is a positive definite solution to
		\begin{equation*}
			P = Q + K^T R K + \prn{\Astar + \Bstar K}^T P \prn{\Astar + \Bstar K},
		\end{equation*}
		and thus does not depend on $W$.		
		
		Now, for $1 \le t < \tWarmup$ we have that $x_{t+1} = \prn{\Astar + \Bstar K_0} x_t + \prn{\Bstar \uNoise + w_t}$, i.e., this is equivalent to an LQ system $(\Astar,\Bstar)$ with noise covariance $\noiseStd^2\prn{I + \Bstar \Bstar^T} \preceq \prn{1 + \systemBound^2} \noiseStd^2 I$ and controller $K_0$ and so we have that
		\begin{align*}
			\Jof{K_0, \noiseStd^2 \prn{I + \Bstar \Bstar^T}}
			=
			\tr{\noiseStd^2 \prn{I + \Bstar \Bstar^T} P}
			\le
			\prn{1 + \systemBound^2} \tr{\noiseStd^2 P}
			=
			\prn{1 + \systemBound^2} \Jof{K_0, \noiseStd^2}
			=
			\prn{1 + \systemBound^2} \Jof{K_0}.
		\end{align*}
		With the above in mind, we bound the first term in the decomposition of $R_3$ using \cref{lemma:steadyStateCostConvergence}. We get that
		\begin{equation} \label{eq:R31B}
		\begin{aligned}
			\EEBrk{\sum_{t=1}^{\tBase - 1} c_t}
			&\le
			\tBase \Jof{K_0, \noiseStd^2 \prn{I + \Bstar \Bstar^T}}
			+ \frac{2 \costMatUpper \kappa_0^4}{\gamma_0} \norm{x_1}^2 \\
			&\le
			\prn{1 + \systemBound^2} \Jof{K_0} \tBase
			.
		\end{aligned}
		\end{equation}
		
		Next, recall that $\gamma_0^{-1} = 2\kappa_0^2$, denote the filtration of the history, $\mathcal{F}_t = \sigma\prn{x_1, u_1, w_1, \ldots, x_t, u_t, w_t}$ and similarly apply \cref{lemma:steadyStateCostConvergence} to get that
		\begin{align*}
			\EEBrk{\sum_{t=\ti}^{\ti[i+1] - 1} c_t ~\bigg|~ \mathcal{F}_{\ti - 1}}
			\le
			(1 + \systemBound^2) \Jof{K_0} \prn{\ti[i+1] - \ti}
			+
			4 \costMatUpper \kappa_0^6 \norm{x_{\ti}}^2.
		\end{align*}
		Now, using \cref{lemma:singleControlBound,lemma:expectedMaxNoise} we get that
		\begin{align*}
			\EEBrk{\indEvent{\nStart > i} \norm{x_{\ti}}^2}
			\le
			\EEBrk{\frac{\kappa_0^2}{\gamma_0^2} \max_{1 \le t \le T} \norm{w_t + \Bstar \uNoise}^2}
			\le
			20 d (1 + \systemBound^2) \noiseStd^2 \kappa_0^8 \log 3 T.
		\end{align*}
		Combining the last two inequalities and noticing that $\indEvent{\nStart > i}$ is $\mathcal{F}_{\ti-1}$ measurable we further have that
		\begin{align}
			\nonumber
			\EEBrk{\indEvent{\nStart > i} \sum_{t=\ti}^{\ti[i+1]-1} c_t}
			&=
			\EEBrk{\indEvent{\nStart > i} \EEBrk{\sum_{t=\ti}^{\ti[i+1]-1} c_t ~\big|~ \mathcal{F}_{\ti - 1}}} \\
			\label{eq:R32B}
			&\le
			\EEBrk{\indEvent{\nStart > i} \prn{
					(1 + \systemBound^2) \Jof{K_0} \prn{\ti[i+1] - \ti} 
					+
					4 \costMatUpper \kappa_0^6 \norm{x_{\ti}}^2
				}} \\
			\nonumber
			&\le
			(1 + \systemBound^2)\prn{
				\PP{\nStart > i} \Jof{K_0} \prn{\ti[i+1] - \ti}
				+
				80  \costMatUpper d \noiseStd^2 \kappa_0^{14} \log 3 T
			}.
		\end{align}		
		
		Now, from \cref{lemma:warmupLengthB} we know that
		$
			\PP{\nStart > 2 + \max\brc{0, \log_2 \frac{\KstarLowerBoundGuess}{\KstarLowerBound}}}
			\le
			\PP{\goodEventB^c}
			\le
			T^{-2}
			,
		$
		and recalling that $\ti = \tBase 4^i$ we get that
		\begin{equation} \label{eq:R33B}
		\begin{aligned}
			\tBase + \sum_{i=0}^{\nT} \prn{\ti[i+1] - \ti} \PP{\nStart > i}
			&\le
			\tBase
			+
			\sum_{i=0}^{\floor{2 + \max\brc{0, \log_2 \frac{\KstarLowerBoundGuess}{\KstarLowerBound}}}} \prn{\ti[i+1] - \ti}
			+
			\sum_{i=0}^{\nT} \prn{\ti[i+1] - \ti} T^{-2} \\
			&=
			\tBase 4^{\floor{3 + \max\brc{0, \log_2 \frac{\KstarLowerBoundGuess}{\KstarLowerBound}}}}
			+
			\prn{\ti[\nT+1] - \tBase} T^{-2} \\
			&\le
			64 \tBase \max\brc{1, \frac{\KstarLowerBoundGuess^2}{\KstarLowerBound^2}}
			+
			4 T^{-1}.
		\end{aligned}
		\end{equation}
		Finally, combining \cref{eq:R31B,eq:R32B,eq:R33B} we get that
		\begin{align*}
			R_3
			&\le
			(1 + \systemBound^2)\prn{
				\Jof{K_0} \prn{
					\tBase
					+
					\sum_{i=0}^{\nT} \prn{\ti[i+1] - \ti} \PP{\nStart > i}
				}
				+
				80  \costMatUpper d \noiseStd^2 \kappa_0^{14} \prn{\nT + 1} \log 3 T
			} \\
			&\le
			(1 + \systemBound^2)\prn{
				64 \Jof{K_0} \max\brc{1, \frac{\KstarLowerBoundGuess^2}{\KstarLowerBound^2}} \tBase
				+
				4 \Jof{K_0} T^{-1}
				+
				80  \costMatUpper d \noiseStd^2 \kappa_0^{14} \log^2 3 T
			} \\
			&\le
			(1 + \systemBound^2)\prn{
				65 \Jof{K_0} \max\brc{1, \frac{\KstarLowerBoundGuess^2}{\KstarLowerBound^2}} \tBase
				+
				80  \costMatUpper d \noiseStd^2 \kappa_0^{14} \log^2 3 T
			},
		\end{align*}
		where the second transition also used $\nT + 1 \le \log 3T$.
	\end{proof}

\section{Lower Bound Proofs} 
\label{sec:lb-proofs}

The next lemma requires the following well known results in LQRs (see, e.g., \citealp{bertsekas1995dynamic}). Consider the Q-function of the system with respect to $k_\star$, that in the one-dimensional case takes the form 
$F(x,u) = x^2 + u^2 + (ax+bu)^2 p_\star.$ Using the form of $k_\star$ given in \cref{eq:Kopt}, and by simple algebra we obtain
\begin{align} \label{eq:1d-bellman}
F(x_t,u_t) - F(x_t,k_\star x_t) 
=
(1 + b^2 p_\star) (u_t - k_\star x_t)^2
.
\end{align}
Further, we have $F(x_t,k_\star x_t) = x_t^2 p_\star$ as both sides are equal to the value of the optimal policy $k_\star$ starting from state $x_t$. Finally, also recall that $J(k_\star) = \sigma^2 p_\star$. The following explains \cref{eq:1d-bellman}:
\begin{align*}
	F(x_t, u_t)
	&=
	x_t^2 + ((u_t - k_\star x_t) + k_\star x_t)^2 + ((a + b k_\star) x_t + b (u_t - k_\star x_t))^2 p_\star \\
	&=
	F(x_t, k_\star x_t)
	+ (u_t - k_\star x_t)^2 + 2(u_t - k_\star x_t)k_\star x_t
	+ b^2 p_\star (u_t - k_\star x_t)^2 + 2bp_\star (u_t - k_\star x_t)(a + bk_\star)x_t \\
	&=
	F(x_t, k_\star x_t)
	+ (1 + b^2 p_\star)(u_t - k_\star x_t)^2
	+ 2 x_t(u_t - k_\star x_t) (k_\star + b p_\star(a + bk_\star)) \\
	&=
	F(x_t, k_\star x_t)
	+ (1 + b^2 p_\star)(u_t - k_\star x_t)^2
	+ 2 x_t(u_t - k_\star x_t) (k_\star(1 + b^2 p_\star) + b p_\star a) \\
	&=
	F(x_t, k_\star x_t)
	+ (1 + b^2 p_\star)(u_t - k_\star x_t)^2,
\end{align*}
where the last transition used $k_\star(1 + b^2 p_\star) = - b p_\star a$ (see \cref{eq:Kopt}).

\begin{lemma} \label{lem:regret-representation}
The expected regret can be written as
\begin{align*}
    \EEBrk[0]{R_T}
    =
    \EEBrk[*]{ \sum_{t=1}^T (1 + b^2 p_\star) (u_t - k_\star x_t)^2 }
    - \EEBrk[!]{ x_{T+1}^2 p_\star }
    .    
\end{align*}
\end{lemma}

\begin{proof}
Using the expressions for the Q-function of the system with respect to $k_\star$, we have that
\begin{align*}
    R_T
    &=
    \sum_{t=1}^T \EEBrk{ x_t^2 + u_t^2 - J(k_\star) }
    \\
    &=
    \sum_{t=1}^T \EEBrk{ F(x_t,u_t) - \big((ax_t+bu_t)^2+w_t^2\big) p_\star }
    \tag{since $J(k_\star) = \EE{[w_t^2 p_\star]}$}
    \\
    &=
    \sum_{t=1}^T \EEBrk{ F(x_t,u_t) - x_{t+1}^2 p_\star }
    \\
    &=
    \sum_{t=1}^T \EEBrk{ F(x_t,u_t)-F(x_t,k_\star x_t) }
        + \sum_{t=1}^T \EEBrk{ x_t^2 p_\star - x_{t+1}^2 p_\star }
    \tag{since $F(x_t,k_\star x_t) = x_t^2 p_\star$} \\
    &=
    \EEBrk{\sum_{t=1}^{T}(1 + b^2 p_\star) (u_t - k_\star x_t)^2}
    +
    \EEBrk{x_1^2 p_\star} - \EEBrk{x_{T+1}^2 p_{\star}}
    \tag{using \cref{eq:1d-bellman}}
    .
\end{align*}
The lemma now follows from our assumption that $x_1 = 0$.
\end{proof}

\begin{lemma} \label{lem:x-norm-bound}
We have
$
    \EEBrk[0]{x_{T+1}^2}
    \le 
    \frac{5}{2}\prn{b^2 \sum_{t=1}^T \EE{[(u_t - k_\star x_t)^2]} + \sigma^2}
    .
$
\end{lemma}

\begin{proof}
Denote $m = a+b k_\star$ and $v_t = u_t - k_\star x_t$ for all $t \geq 1$. 
Then,
$
    x_{t+1} 
    =
    a x_t + b (u_t - k_\star x_t + k_\star x_t) + w_t
    =
    m x_t + b v_t + w_t
    ,
$
and by unfolding the recursion and using $x_1 = 0$ we obtain
\[
    x_{T+1} 
    = 
    \sum_{t=1}^{T} m^{T-t} b v_t + \sum_{t=1}^{T} m^{T-t} w_t
    ,
\]
hence
\begin{align*}
    \EE{[x_{T+1}^2]}
    \le
    2b^2 \EE{\prn{\sum_{t=1}^{T} m^{T-t} v_t}^2}
    +
    2 \EE{\prn{\sum_{t=1}^{T} m^{T-t} w_t}^2}
    ,
\end{align*}
Now, observe that
$$
    |m| 
    = 
    |a + b k_\star|
    = 
    \Big| a - b \cdot \frac{a b p_\star}{1 + b^2 p_\star} \Big|
    =
    \Big| \frac{a}{1 + b^2 p_\star} \Big|
    \leq
    |a|
    \leq
    \frac{1}{\sqrt{5}}
    .
$$
Using this bound and the Cauchy-Schwartz inequality, we have
\begin{align*}
    \EE{\prn{\sum_{t=1}^{T} m^{T-t} v_t}^2}
    \le
    \sum_{t=1}^T m^{2(T-t)} \cdot \EEBrk{\sum_{t=1}^T v_t^2}
    &\le
    \frac{1}{1-m^2} \EEBrk{\sum_{t=1}^T v_t^2}
    \le
    \frac{5}{4} \EEBrk{\sum_{t=1}^T v_t^2}
    .
\end{align*}
Further, as the noise terms $w_1,\ldots,w_T$ are i.i.d. and have variance $\sigma^2$,
\begin{align*}
    \EE{\prn{\sum_{t=1}^{T} m^{T-t} w_t}^2}
    =
    \sum_{t=1}^{T} m^{2(T-t)} \EE{[w_t^2]}
    \leq
    \frac{1}{1-m^2} \sigma^2
    \leq
    \frac{5}{4} \sigma^2
    .
\end{align*}
Combining inequalities, the lemma follows.
\end{proof}

\begin{proof}[of \cref{lem:lb1}]
Since $1 + b^2 p_\star \geq 1$ and $p_\star \leq 5 / 4$ (see \cref{eq:lb-kp-bounds}), \cref{lem:regret-representation} lower bounds the regret as
\begin{align*}
    \EEBrk[0]{R_T}
    \geq
    \EEBrk[3]{ \sum_{t=1}^T (u_t - k_\star x_t)^2 } - \frac{5}{4}\EEBrk[0]{x_{T+1}^2}
    .
\end{align*}
Plugging in the bound of \cref{lem:x-norm-bound} and the assumption that $b^2 = \epsilon \leq 1/400$, we obtain
\begin{align} \label{eq:lb1}
    \EEBrk[0]{R_T}
    \geq
    \frac{99}{100} \EEBrk[3]{ \sum_{t=1}^T (u_t - k_\star x_t)^2 } - 4 \sigma^2
    .
\end{align}
On the other hand, note that $u_t^2 \le 2 (u_t - k_\star x_t)^2 + 2 k_\star^2 x_t^2$, and so 
\begin{align*}
	\EEBrk[3]{ \sum_{t=1}^T u_t^2 }
	\le 
	2 \EEBrk[3]{ \sum_{t=1}^{T} (u_t - k_{\star} x_{t})^2 } + 2k_\star^2 \EEBrk[3]{ \sum_{t=1}^{T} x_{t}^2 }
	.
\end{align*}
Further, since $J(k_\star) = \sigma^2 p_\star \leq \frac{5}{4}\sigma^2$ we have
\begin{align*}
    \EE{\Bigg[\sum_{t=1}^T x_t^2\Bigg]}
    \leq
    \EE{\Bigg[\sum_{t=1}^T (x_t^2 + u_t^2)\Bigg]}
    =
    \EE{[R_T]} + T \EE{[J(k_\star)]}
    \leq
    \EE{[R_T]} + \frac{5}{4}\sigma^2 T
    .
\end{align*}
Therefore,
\begin{align} \label{eq:lb2}
    \EE{\Bigg[ \sum_{t=1}^T u_t^2 \Bigg]}
    \leq
    2 \EE{\Bigg[ \sum_{t=1}^{T} (u_t - k_{\star} x_{t})^{2} \Bigg]}
    + 2k_\star^2 \EE{[ R_T ]} + \frac{5}{2}\sigma^2 k_\star^2 T
    .
\end{align}
Combining \cref{eq:lb1,eq:lb2} and recalling that $2 k_\star^2 \leq \epsilon \leq 1$ (see \cref{eq:lb-kp-bounds}), results with
\begin{align*}
    \EE{\Bigg[ \sum_{t=1}^T u_t^2 \Bigg]}
    \leq
    2\prn[!]{\frac{100}{99}\EE{[ R_T ]} + 5\sigma^2} + 2k_\star^2 \EE{[ R_T ]} + \frac{5}{2}\sigma^2 k_\star^2 T
    \leq
    3 \EE{[ R_T ]} + \frac{5}{2}\sigma^2 k_\star^2 T + 12\sigma^2
    ,
\end{align*}
and changing sides yields the second part of the lemma, thus concluding the proof.
\end{proof}

\begin{proof}[of \cref{lem:many-large-x}]
Let $Z$ be a standard Gaussian random variable. 
Then, using a standard Gaussian tail lower bound,
\begin{align*}
    \Pr \bigg[\abs{w_{t-1}} \ge \frac{2\sigma}{5} \bigg]
    =
    \Pr \bigg[\abs{Z} \ge \frac{2}{5} \bigg]
    \ge
    \frac{17}{25}.
\end{align*}
Now, recall that $x_t = a x_{t-1} + b u_{t-1} + w_{t-1}$ and notice that, as the learning algorithm is deterministic, both $x_{t-1}$ and $u_{t-1}$ are determined conditioned on $x_1,\ldots,x_{t-1}$. 
We next aim to lower bound $\Pr[|x_t| > 2 \sigma / 5 \mid x_1,\ldots,x_{t-1}]$ which we claim that, as $w_{t-1}$ is a zero-mean Gaussian random variable, is minimized when $a x_{t-1} + b u_{t-1} = 0$.
Therefore,
\begin{align*}
    \Pr \bigg[|x_t| > \frac{2 \sigma}{5} \Bigm| x_1,\ldots,x_{t-1} \bigg] 
    \ge 
    \Pr \bigg[|w_{t-1}| > \frac{2 \sigma}{5} \bigg]
    \ge 
    \frac{17}{25}.
\end{align*}
Denote by $I_t = \indEvent{|x_t| > 2 \sigma/5}$. Then, by Azuma's concentration inequality we have that with probability at least $7/8$,
\begin{align*}
    \sum_{t=1}^T I_t 
    \ge
    \sum_{t=1}^T \EEBrk{I_t \mid x_1,\ldots,x_{t-1}} - \sqrt{\frac{T}{2} \log 8}
    \ge
    \frac{17}{25} T - \sqrt{2T}
    \ge
    \frac{2}{3} T,
\end{align*}
where for the last inequality we used the assumption that $T \ge 12000$.
\end{proof}

\begin{proof}[of \cref{lem:tv-distance1}]
First, using Pinsker's inequality yields
\begin{align} \label{eq:pinsker}
    \TV{\Pr_{+}[x^{(T)}]}{\Pr_{-}[x^{(T)}]}
    \le
    \sqrt{\frac{1}{2} \KL{\Pr_{+}[x^{(T)}]}{\Pr_{-}[x^{(T)}]}}
    ~,
\end{align}
and by the chain rule of the KL divergence
\begin{equation} \label{eq:klchainrule}
    \KL{\Pr_{+}[x^{(T)}]}{\Pr_{-}[x^{(T)}]}
    =
    \sum_{t=1}^T \EEBrk{\KL{\Pr_{+}[x_t \mid x^{(t-1)}]}{\Pr_{-}[x_t \mid x^{(t-1)}]}}
    .
\end{equation}
Next, let $\EE[+]{}$ and $\EE[-]{}$ denote the expectations conditioned on whether $\chi = 1$ or $\chi = -1$ respectively.
Observe that as the learning algorithm is deterministic, the sequence of actions $u_1,\ldots,u_{t-1}$ is determined given~$x^{(t-1)}$. 
As such, given $x^{(t-1)}$, the random variable $x_t$ is Gaussian with variance $\sigma^2$ and expectation $a x_{t-1} + \sqrt{\epsilon} \chi u_{t-1}$. 
Therefore, by a standard formula for the KL divergence between Gaussian random variables, we have
\begin{align*}
    \KL{\Pr_{+}[x_t \mid x^{(t-1)}]}{\Pr_{-}[x_t \mid x^{(t-1)}]}
    &=
    \frac{1}{2\sigma^2} \EE[+]{\big( (a x_{t-1} + \sqrt{\epsilon} u_{t-1}) - (a x_{t-1} - \sqrt{\epsilon} u_{t-1}) \big)^2} \\
    &=
    \frac{1}{2\sigma^2} \EE[+]{\big( 2 \sqrt{\epsilon} u_{t-1} \big)^2} \\
    &=
    \frac{2 \epsilon}{\sigma^2} \EE[+]{[u_{t-1}^2]}
    ,
\end{align*}
unless $t=1$ in which case 
$
    \KL{\Pr_{+}[x_1]}{\Pr_{-}[x_1]} = 0
$
since $x_{1}$ is fixed.
Using this bound in \cref{eq:klchainrule} and substituting into \cref{eq:pinsker} yields
\begin{align*}
    \TV{\Pr_{+}[x^{(T)}]}{\Pr_{-}[x^{(T)}]}
    \leq
    \sqrt{ \frac{\epsilon}{\sigma^2} \EE[+]{\Bigg[ \sum_{t=1}^T u_t^2 \Bigg]} }
    ~.
\end{align*}
Similarly, switching the roles of $\Pr_{+}$ and $\Pr_{-}$, we get the bound
\begin{align*}
    \TV{\Pr_{+}[x^{(T)}]}{\Pr_{-}[x^{(T)}]}
    \leq
    \sqrt{ \frac{\epsilon}{\sigma^2} \EE[-]{\Bigg[ \sum_{t=1}^T u_t^2 \Bigg]} }
    ~.
\end{align*}
Averaging the two inequalities, using the concavity of the square root, and since $\EE{[\cdot]} = \tfrac12 \EE[+]{[\cdot]} + \tfrac12 \EE[-]{[\cdot]}$, we obtain our claim.
\end{proof}

\section{Technical Lemmas}

\subsection{Noise Bounds}

The following theorem is a variant of the Hanson-Wright inequality \citep{hanson1971bound,wright1973bound} which can be found in \citet{hsu2012tail}.
\begin{theorem} \label{thm:hansonWright}
	Let $x \sim \gaussDist{0}{I}$ be a Gaussian random vector,, let $A \in \RR[m \times n]$ and define $\Sigma = A^T A$. Then we have that
	\begin{equation*}
		\PP{\norm{Ax}^2 > \tr{\Sigma} + 2\sqrt{\tr{\Sigma^2}z} + 2 \norm{\Sigma} z}
		\le
		\exp\prn{-z}
		,\qquad
		\text{ for all } z \ge 0.
	\end{equation*}
\end{theorem}

The following lemma is a direct corollary of \cref{thm:hansonWright}.
\begin{lemma} \label{lemma:noiseBound}
	Let $w_t \in \RR[d]$ for $t = 1, \ldots, T$ be i.i.d.~random variables with distribution $\gaussDist{0}{\noiseStd^2 I}$. Suppose that $T > 2$, then with probability at least $1 - \delta$ we have that
	\begin{equation*}
	\max_{1 \le t \le T} \norm{w_t} \le \noiseStd \sqrt{5 d \log \frac{T}{\delta}}
	.
	\end{equation*}
\end{lemma}
\begin{proof}
	Consider \cref{thm:hansonWright} with $A = \noiseStd I$ and thus $\Sigma = \noiseStd^2 I$. We then have that $\tr{\Sigma} = d \noiseStd^2$, $\norm{\Sigma} \le \noiseStd^2$ and $\tr{\Sigma^2} \le \norm{\Sigma} \tr{\Sigma} \le d \noiseStd^4$. We conclude that for $z \ge 1$ we have that
	\begin{equation*}
		\tr{\Sigma} + 2\sqrt{\tr{\Sigma^2}z} + 2 \norm{\Sigma} z
		\le
		\noiseStd^2 d + 2 \noiseStd^2 \sqrt{d z} + 2\noiseStd^2 z
		\le
		5 \noiseStd^2 d z.
	\end{equation*}
	Now, for $x \sim \gaussDist{0}{I}$ we have that $w_t \overset{d}{=} A x$ (equals in distribution). We thus have that for $z \ge 1$
	\begin{align*}
		\PP{\norm{w_t} > \noiseStd \sqrt{5 d z}}
		\le
		\PP{\norm{Ax} > \sqrt{\tr{\Sigma} + 2\sqrt{\tr{\Sigma^2}z} + 2 \norm{\Sigma} z}}
		\le
		\exp \prn{-z}.
	\end{align*}
	Denoting $z = \log \frac{T}{\delta}$, the assumption $T > 2$ ensures that $z \ge 1$ and thus
	$
		\PP{\norm{w_t} > \noiseStd \sqrt{5 d \log \frac{T}{\delta}}}
		\le
		\frac{\delta}{T}
		.
	$
	Performing a union bound over $1 \le t \le T$ we conclude that
	$$
		\PP{\max_{1 \le t \le T}\norm{w_t} > \noiseStd \sqrt{5 d \log \frac{T}{\delta}}}
		\le
		\delta,
	$$
	and taking the complement we obtain the desired.
\end{proof}

\begin{lemma}[Expected maximum noise] \label{lemma:expectedMaxNoise}
	Let $E$ be an event such that $\PP{E} \le \delta$ for some $\delta \in \brk[s]{0, 1}$ and let $w_t \in \RR[d]$ for $t = 1, \ldots, T$ be i.i.d.~ random variables with distribution $\gaussDist{0}{\noiseStd^2 I}$.
	Suppose $T > 2$, then we have that
	\begin{enumerate}
		\item $\EEBrk{\max_{1 \le t \le T} \norm{w_t}^2} \le 5 \noiseStd^2 d \log 3 T$;
		\item $\EEBrk{\indEvent{E} \max_{1 \le t \le T} \norm{w_t}^2}
				\le
				5 \noiseStd^2 d \delta \log \frac{3 T}{\delta}.$
	\end{enumerate}
\end{lemma}
\begin{proof}
	Recall that from \cref{lemma:noiseBound} we have that for all $x \ge 5 \noiseStd^2 d \log T$
	\begin{equation*}
	\PP{\max_{1 \le t \le T}\norm{w_t}^2 > x}
	\le
	T \exp \prn{- \frac{x}{5 \noiseStd^2 d}}.
	\end{equation*}
	Applying the tail sum formula we get that
	\begin{align*}
	\EEBrk{\max_{1 \le t \le T}\norm{w_t}^2}
	&=
	\int_{0}^{\infty} \PP{\max_{1 \le t \le T}\norm{w_t}^2 > x} dx \\
	&\le
	5 \noiseStd^2 d \log T + \int_{5 \noiseStd^2 d \log T}^{\infty} T \exp \prn{- \frac{x}{5 \noiseStd^2 d}} dx \\
	&\le
	5 \noiseStd^2 d \log 3 T,
	\end{align*}
	proving the first part of the lemma. For the second part notice that
	$
		\PP{\indEvent{E} \max_{1 \le t \le T} \norm{w_t}^2 > x}
		\le \min\brc{\PP{E}, \PP{\max_{1 \le t \le T} \norm{w_t}^2 > x}}
		.
	$
	So, applying the tail sum formula we get that
	\begin{align*}
		\EEBrk{\indEvent{E}\max_{1 \le t \le T}\norm{w_t}^2}
		&=
		\int_{0}^{\infty} \PP{\indEvent{E}\max_{1 \le t \le T}\norm{w_t}^2 > x} dx \\
		&\le
		\int_{0}^{5\noiseStd^2 d \log \frac{T}{\delta}} \PP{E} dx + \int_{5\noiseStd^2 d \log \frac{T}{\delta}}^{\infty} \PP{\max_{1 \le t \le T}\norm{w_t}^2 > x} dx \\
		&\le
		5 \noiseStd^2 d \delta \log \frac{T}{\delta} + \int_{5 \noiseStd^2 d \log \frac{T}{\delta}}^{\infty} T \exp \prn{- \frac{x}{5 \noiseStd^2 d}} dx \\
		&=
		5 \noiseStd^2 d \delta \prn{1 + \log \frac{T}{\delta}} \\
		&\le
		5 \noiseStd^2 d \delta \log \frac{3 T}{\delta},
	\end{align*}
	proving the second part and concluding the proof.
\end{proof}

\subsection{Estimation auxiliary lemmas}
The following is due to \citet{cohen2019learning}. Here we state the result for a general sequence of conditionally Gaussian vectors but the proof follows without change.
\begin{lemma}[Theorem 20 of \citealp{cohen2019learning}] \label{lemma:VtLowerBound}
	Let $z_t$ for $t=1, 2, \ldots$ be a sequence random variables that is adapted to a filtration $\seqDef{\mathcal{F}_t}{t=1}{\infty}$. Suppose that $z_t$ are conditionally Gaussian on $\mathcal{F}_{t-1}$ and that $\EEBrk{z_t z_t^T ~\big| \mathcal{F}_{t-1}} \succeq \sigma_z^2 I$ for some fixed $\sigma_z^2 > 0$. Then for $t \ge 200 d \log \frac{12}{\delta}$ we have that with probability at least $1 - \delta$
	\begin{equation*}
	\sum_{s=1}^{t} z_s z_s^T \succeq \frac{t \sigma_z^2}{40}I.
	\end{equation*}
\end{lemma}

\begin{lemma} \label{lemma:logDetBound}
	Let $z_s \in \RR[m]$ for $s = 1, \ldots, t-1$ be such that $\norm{z_s}^2 \le \lambda$. Define $V_t = \lambda I + \sum_{s=1}^{t-1} z_s z_s^T$ then we have that
	\begin{equation*}
	    \log \frac{\det\prn{V_t}}{\det\prn{V_1}} 
	    \leq 
	    m \log t
	.
	\end{equation*}
\end{lemma}
\begin{proof}
	First we have that
	\begin{equation*}
		\norm{V_t}
		\le
		\lambda + \sum_{s=1}^{t-1} \norm{z_s z_s^T}
		=
		\lambda + \sum_{s=1}^{t-1} \norm{z_s}^2
		\le
		\lambda t.
	\end{equation*}
	Now, recall that $\det\prn{V_t} \le \det \prn{\norm{V_t}^m}$ and so we have that
	\begin{align*}
		\log \frac{\det\prn{V_t}}{\det\prn{V_1}}
		\le
		\log \frac{\det\prn{\norm{V_t}^m}}{\lambda^m}
		\le
		\log \frac{\lambda^m t^m}{\lambda^m}
		=
		m \log t,
	\end{align*}
	as desired.
\end{proof}

\subsection{Strong Stability Lemmas}

The following lemma bounds the norm of the state when playing a strongly stable controller. Its proof adapts techniques from \cite{cohen2019learning}.
\begin{lemma} \label{lemma:singleControlBound}
	Suppose $K$ is a $(\kappa, \gamma)-$strongly stable controller and $s_0, s_1$ are integers such that $1 \le s_0 < s_1 \le T$. Let $x_s$ for $s = s_0, \ldots s_1$ be the sequence of states generated under the control $K$ starting from $x_{s_0}$, i.e., $x_{s+1} = \prn{\Astar + \Bstar K} x_s + w_s$ for all $s_0 \le s < s_1$. Then we have that
	\begin{equation*}
	\norm{x_t}
	\le
	\kappa(1 - \gamma)^{t - s_0}\norm{x_{s_0}} + \frac{\kappa}{\gamma} \max_{1 \le t \le T}\norm{w_t},
	\qquad \text{for all } s_0 \le t \le s_1.
	\end{equation*} 
\end{lemma}
\begin{proof}
	Denote $M = \Astar + \Bstar K$ then for $s_0 < t \le s_1$ we have that $x_t = M x_{t-1} + w_{t-1}$ and by expanding this equation we have
	\begin{equation*}
	x_t = M^{t - s_0} x_{s_0} + \sum_{s=s_0}^{t-1} M^{t - (s + 1)}w_s.
	\end{equation*}
	Recall that by strong stability we have that
	\begin{equation*}
	\norm{M^s}
	=
	\norm{H L^s H^{-1}}
	\le
	\kappa (1 - \gamma)^s.
	\end{equation*}
	To ease notation denote $W = \max_{1 \le t \le T} \norm{w_t}$. Then for $s_0 < t \le s_1$ we have that
	\begin{align*}
	\norm{x_t}
	&\le
	\norm{M^{t - s_0}} \norm{x_{s_0}} + \sum_{s=s_0}^{t-1} \norm{M^{t - (s + 1)}} \norm{w_s} \\
	&\le
	\kappa(1-\gamma)^{t - s_0} \norm{x_{s_0}} + \sum_{s=s_0}^{t-1} \kappa(1 - \gamma)^{t - (s + 1)} W \\
	&\le
	\kappa (1-\gamma)^{t - s_0} \norm{x_{s_0}} +\frac{\kappa}{\gamma} W
	.
	\qedhere
	\end{align*}
\end{proof}

The following lemma bounds the norm of the state when playing a sequence of strongly stable controllers.
\begin{lemma}
\label{lemma:multiControlBound}
	Suppose $K_1, \ldots, K_l$ are $(\kappa,\gamma)$-strongly stable controllers and $\seqDef{t_i}{i=1}{l+1}$ are integers such that $1 \le t_1 < \ldots < t_{l+1} \le T$.
	Let $x_t$ for $t = t_1, \ldots t_{l+1}$ be the sequence of states generated by starting from $x_{t_1}$ and playing controller $K_i$ at times $t_i \le t < t_{i+1}$, i.e., $x_{t+1} = \prn{\Astar + \Bstar K_i} x_t + w_t$ for all $t_i \le t < t_{i+1}$. Denote $\tau = \min_{i} \brc{t_{i+1} - t_i}$ and suppose that $\tau \ge \gamma^{-1}\log(2\kappa)$, then we have that
	\begin{equation*}
	\norm{x_t}
	\le
	3\kappa
	\max\brc{
		\frac12 \norm{x_{t_1}},
		\frac{\kappa}{\gamma} \max_{1 \le t \le T}\norm{w_t}},
	\quad \forall ~ t_1 \le t \le t_{l+1}.
	\end{equation*}
\end{lemma}

\begin{proof}
	For $0 < \gamma \le 1$ it is a well known fact that $\gamma \le - \log 1 - \gamma$. Plugging this into the lower bound on $\tau$ and rearranging we get that $\kappa (1 - \gamma)^\tau \le \frac{1}{2}$.
	Now, applying \cref{lemma:singleControlBound} with $s_0 = t_i$ and $s_1 = t_{i+1}$, and taking $t = t_{i+1}$ we have that
	\begin{align*}
	\norm{x_{t_{i+1}}}
	&\le
	\kappa (1 - \gamma)^{t_{i+1} -  t_i} \norm{x_{t_i}} + \frac{\kappa}{\gamma} W \\
	&\le
	\kappa (1 - \gamma)^{\tau} \norm{x_{t_i}} + \frac{\kappa}{\gamma} W \\
	&\le
	\frac{1}{2} \norm{x_{t_i}} + \frac{\kappa}{\gamma} W,
	\end{align*}
	and solving this difference equation we get that
	\begin{equation*}
	\norm{x_{t_i}}
	\le
	\frac{2\kappa}{\gamma}W + \prn{\norm{x_{t_1}} - \frac{2\kappa}{\gamma}W} 2^{1-i}
	\le
	\max\brc{\norm{x_{t_1}}, \frac{2\kappa}{\gamma}W}.
	\end{equation*}
	Plugging this result back into \cref{lemma:singleControlBound} we have that for $t_i < t \le t_{i+1}$
	\begin{align*}
	\norm{x_t}
	&\le
	\kappa (1-\gamma)^{t - t_i}\max\brc{\norm{x_{t_1}}, \frac{2\kappa}{\gamma}W} + \frac{\kappa}{\gamma} W \\
	&\le
	\kappa \max\brc{\norm{x_{t_1}}, \frac{2\kappa}{\gamma}W} + \frac{\kappa}{\gamma} W \\
	&\le
	\kappa \max\brc{\frac{3\norm{x_{t_1}}}{2}, \frac{3\kappa}{\gamma}W},
	\end{align*}
	where the last inequality used the fact that $\kappa \ge 1$. This is true for all $i$ and thus for all $t_1 \le t \le t_{l+1}$.
\end{proof}

The next two lemmas require the following well known result in linear control theory (see, e.g., \citealp{bertsekas1995dynamic}). We have that $\Jof{K} = \noiseStd^2 \tr{P}$ where $P$ is a positive definite solution of
\begin{equation} \label{eq:Pbellman}
P = Q + K^T R K + \prn{\Astar +\Bstar K}^T P \prn{\Astar + \Bstar K}.
\end{equation}

The following lemma relates the expected cost of playing controller $K$ for $t$ rounds to the infinite horizon cost of $K$.
\begin{lemma}
\label{lemma:steadyStateCostConvergence}
	Suppose $K$ is a $(\kappa,\gamma)-$strongly stable controller and let $x_s$ for $s = 1, \ldots t$ be the sequence of states generated under the control $K$ starting from $x_1$, i.e., $x_{s+1} = \prn{\Astar + \Bstar K} x_s + w_s$ for all $1 \le s < t$. Then we have that
	\begin{equation*}
	\EEBrk{\sum_{s=1}^t x_s^T \prn{Q + K^T R K} x_s ~\bigg|~ x_1}
	\leq
	t \Jof{K}
	+
	\frac{2 \costMatUpper \kappa^4}{\gamma} \norm{x_1}^2
	.
	\end{equation*}
\end{lemma}

\begin{proof}
	To ease notation, assume, without loss of generality, that $x_1$ is deterministic. We thus omit the conditioning on $x_1$ in all expectation arguments.
	
	First, recall that $x_{s+1} = \prn{\Astar + \Bstar K} x_s + w_s$ and $\Jof{K} = \noiseStd^2 \tr{P}$ where $P$ satisfies \cref{eq:Pbellman}. Then we have that
	\begin{align*}
		\EEBrk{x_{s+1}^T P x_{s+1}}
		&=
		\EEBrk{\prn{\prn{\Astar + \Bstar K} x_s + w_s}^T P \prn{\prn{\Astar + \Bstar K} x_s + w_s}} \\
		&=
		\EEBrk{\prn{\prn{\Astar + \Bstar K} x_s}^T P \prn{\prn{\Astar + \Bstar K} x_s}}
		+
		\EEBrk{w_s^T P w_s} \\
		&=
		\EEBrk{x_s^T \prn{\Astar + \Bstar K}^T P \prn{\Astar + \Bstar K} x_s}
		+
		\Jof{K}.
	\end{align*}
	Now, multiplying \cref{eq:Pbellman} by $x_s$ from both sides and taking expectation we get that
	\begin{align*}
		\EEBrk{x_s^T P x_s}
		&=
		\EEBrk{x_s^T \prn{Q + K^T R K} x_s}
		+
		\EEBrk{x_s^T \prn{\Astar + \Bstar K}^T P \prn{\Astar + \Bstar K} x_s} \\
		&=
		\EEBrk{x_s^T \prn{Q + K^T R K} x_s}
		+
		\EEBrk{x_{s+1}^T P x_{s+1}}
		-
		\Jof{K},
	\end{align*}
	and changing sides and summing over $s$ we get that
	\begin{align*}
		\EEBrk{x_{1}^T P x_{1}
			-
			x_{t+1}^T P x_{t+1}}
		=
		\sum_{s=1}^{t} \EEBrk{x_s^T P x_s - x_{s+1}^T P x_{s+1}}
		=
		\EEBrk{\sum_{s=1}^{t} x_s^T \prn{Q + K^T R K} x_s} - t \Jof{K},
	\end{align*}
	and changing sides again we conclude that
	\begin{equation*}
		\EEBrk{\sum_{s=1}^{t} x_s^T \prn{Q + K^T R K} x_s}
		\le
		t \Jof{K} + \EEBrk{x_{1}^T P x_{1}}
		\le
		t \Jof{K} + \norm{x_1}^2 \norm{P}.
	\end{equation*}
	We conclude the proof by bounding $\norm{P}$. To that end,
	recall that the strong stability of $K$ implies that $\Astar + \Bstar K = H L H^{-1}$ where $\norm{L} \le 1- \gamma$ and $\norm{H} \norm{H^{-1}} \le \kappa$. Applying \cref{eq:Pbellman} recursively we then have that
	\begin{align*}
		\norm{P}
		&=
		\norm*{\sum_{s=0}^{\infty} \prn{\prn{\Astar + \Bstar K}^s}^T \prn{Q + K^T R K} \prn{\Astar + \Bstar K}^s} \\
		&=
		\norm*{\sum_{s=0}^{\infty} \prn{H L^s H^{-1}}^T \prn{Q + K^T R K} H L^s H^{-1}} \\
		&\le
		\norm!{H}^2 \norm!{H^{-1}}^2 \norm!{Q + K^T R K} \sum_{s=0}^{\infty} \norm!{L}^{2s} \\
		&\le
		2 \costMatUpper \kappa^4 \sum_{s=0}^{\infty} \prn{1 - \gamma}^s
		=
		\frac{2 \costMatUpper \kappa^4}{\gamma}
		,
	\end{align*}
	thus concluding the proof.	
\end{proof}

The following lemma relates the infinite horizon cost of a controller to its strong stability parameters. Its proof is an adaptation of Lemma 18 in \cite{cohen2019learning} that fits our assumptions.
\begin{lemma} \label{lemma:costToStability}
	Suppose $\Jof{K} < J$ then $K$ is $(\kappa, \gamma)-$strongly stable with $\kappa = \sqrt{\frac{J}{\costMatLower \noiseStd^2}}$ and $\gamma = \frac{\costMatLower \noiseStd^2}{2 J}$.
\end{lemma}
\begin{proof}
	Recall that $\Jof{K} = \noiseStd^2 \tr{P}$ where $P$ satisfies \cref{eq:Pbellman}.
	Using the bound $\Jof{K} \le J$ we have that $\tr{P} \le J / \noiseStd^2$ and thus also that $P \preceq (J / \noiseStd^2) I$. Recalling that $Q \succeq \costMatLower I$ we get that $Q \succeq \frac{\costMatLower \noiseStd^2}{J} P = 2\gamma P$. Recalling that $R$ is positive definite and plugging back into \cref{eq:Pbellman} we get that
	\begin{align*}
	P
	\succeq
	2\gamma P + \prn{\Astar + \Bstar K}^T P \prn{\Astar + \Bstar K},
	\end{align*}
	rearranging the equation we get that
	\begin{equation*}
	P^{-1/2}\prn{\Astar + \Bstar K}^T P \prn{\Astar + \Bstar K} P^{-1/2}
	\preceq
	\prn{1 - 2 \gamma}I.
	\end{equation*}
	Now, denote $H = P^{-1/2}$ and $L = P^{1/2}\prn{\Astar + \Bstar K} P^{-1/2}$ and notice that indeed $H L H^{-1} = \Astar + \Bstar K$. Plugging into the above we get that
	\begin{align*}
	P^{-1/2}\prn{\Astar + \Bstar K}^T P \prn{\Astar + \Bstar K} P^{-1/2}
	=
	H \prn{H L H^{-1}}^T H^{-1} H^{-1} \prn{H L H^{-1}} H
	=
	L^T L
	\preceq
	\prn{1 - 2 \gamma}I,
	\end{align*}
	and thus $\norm{L} \le \sqrt{1 - 2 \gamma} \le 1 - \gamma$. Now recall that $P \preceq (J / \noiseStd^2) I$ and thus $\norm{H^{-1}} = \norm{P^{1/2}} \le \sqrt{J / \noiseStd^2}$. Going back to \cref{eq:Pbellman} we also have that $P \succeq Q \succeq \costMatLower I$ and thus $\norm{H} = \norm{P^{-1/2}} \le \sqrt{1 / \costMatLower}$. All together, we get that $\norm{H}\norm{H^{-1}} \le \sqrt{J / \costMatLower \noiseStd^2} = \kappa$. Finally, recall that $R \succeq \costMatLower I$ and thus going back to \cref{eq:Pbellman} we have that $P \succeq K^T R K \succeq \costMatLower K^T K$ and thus $\norm{K} \le \sqrt{\norm{P} / \costMatLower} \le \sqrt{J / \costMatLower \noiseStd^2} = \kappa$, as desired.
\end{proof}

The following lemma relates system parameter estimation bounds to properties of the resulting greedy controller.
\begin{lemma} \label{lemma:goodController}
	Let $A \in \RR[d \times d] ,B \in \RR[d \times K]$ and denote $\Delta = \max\brc{\norm{A - \Astar}, \norm{B - \Bstar}}$. Taking $K = \KoptOf{A}{B}$ and denoting $\kappa = \sqrt{\frac{\optCostBound + \rechtConst \rechtEps^2}{\costMatLower \noiseStd^2}}$ and $\gamma = \frac{1}{2 \kappa^2}$ we have that
	\begin{enumerate}
		\item If $\Delta \le \rechtEps$ then $K$ is $(\kappa,\gamma)-$strongly stable;
		\item If $\Delta \le \min\brc{\rechtEps, \frac{\mu}{4 \kappa \rechtConst}}$ then $K K^T \succeq \Kstar \Kstar^T - \frac{\mu}{2}I$ and $\Kstar \Kstar^T \succeq K K^T - \frac{\mu}{2}I$;
		\item If $\Delta \le \min\brc{\rechtEps, \frac{\KstarLowerBound}{4 \kappa \rechtConst}}$ then $K K^T \succeq \frac{\KstarLowerBound}{2} I$.
	\end{enumerate}
\end{lemma}
\begin{proof}
	First, if $\Delta \le \rechtEps$ we can invoke \cref{lemma:recht} to get that $\Jof{K} \le \Jstar + \rechtConst \rechtEps^2 \le \optCostBound + \rechtConst \rechtEps^2$ and so by \cref{lemma:costToStability}, $K$ is $(\kappa,\gamma)-$strongly stable, proving the first part of the lemma.
	
	Second, if $\Delta \le \min\brc{\rechtEps, \frac{\mu}{4 \kappa \rechtConst}}$ then we can invoke \cref{lemma:recht} to get that $\norm{K - \Kstar} \le \frac{\mu}{4 \kappa}$. Moreover, by the first claim of the lemma, $K, \Kstar$ are $(\kappa,\gamma)-$strongly stable and thus upper bounded by $\kappa$. Combining the above we get that
	\begin{align*}
	K K^T
	&=
	\Kstar \Kstar^T - \frac12 \prn{
		\prn{\Kstar + K} \prn{\Kstar - K}^T
		+
		\prn{\Kstar - K} \prn{\Kstar + K}^T
	} \\
	&\succeq
	\Kstar \Kstar^T - \prn{\norm{\Kstar} + \norm{K}} \norm{\Kstar - K} I \\
	&\succeq
	\Kstar \Kstar^T - \frac{2 \kappa\mu}{4 \kappa} I
	=
	\Kstar \Kstar^T - \frac{\mu}{2} I,
	\end{align*}
	and reversing the roles of $K$ and $\Kstar$ in the above yields $\Kstar \Kstar^T \succeq K K^T - \frac{\mu}{2} I$, thus proving the second part of the lemma.
	
	Finally, if $\Delta \le \min\brc{\rechtEps, \frac{\KstarLowerBound}{4 \kappa \rechtConst}}$, then recalling that $\Kstar \Kstar \succeq \KstarLowerBound I$ and continuing from the second part we get that
	\begin{equation*}
	K K^T
	\succeq
	\Kstar \Kstar^T - \frac{\KstarLowerBound}{2} I
	\succeq
	\KstarLowerBound I - \frac{\KstarLowerBound}{2} I
	=
	\frac{\KstarLowerBound}{2} I,
	\end{equation*}
	thus concluding the third and final part of the lemma.
\end{proof}

\end{document}